\documentclass[conference]{IEEEtran}
\usepackage{times}

\usepackage[numbers]{natbib}
\usepackage{multicol}
\usepackage[bookmarks=true]{hyperref}
\usepackage{multirow}

\usepackage{hyperref}
\usepackage{url}
\usepackage{float}
\usepackage{adjustbox}
\usepackage{siunitx}
\usepackage{colortbl}
\usepackage{tabularx}
\usepackage{comment}
\usepackage{mathtools}
\usepackage{amsfonts}  %
\usepackage{cleveref}
\usepackage{graphics}
\usepackage{color}
\usepackage{dsfont}
\usepackage[]{mdframed}
\usepackage{algorithm}
\usepackage{algorithmic}
\usepackage{xparse}
\usepackage{amsmath}
\usepackage{bm}
\usepackage{mathtools}
\usepackage{amssymb}
\usepackage{amsthm}
\usepackage{amssymb}
\usepackage{fvextra}

\usepackage{booktabs}
\usepackage{epigraph}
\usepackage{listings}
\usepackage{xcolor} %
\usepackage{subcaption}

\newtheorem{theorem}{Theorem}

\newtheorem{lemma}{Lemma}

\definecolor{codegreen}{rgb}{0,0.6,0}
\definecolor{codegray}{rgb}{0.5,0.5,0.5}
\definecolor{codepurple}{rgb}{0.58,0,0.82}
\definecolor{backcolour}{rgb}{0.95,0.95,0.92}
\definecolor{promptcolor}{HTML}{D1D0F2}
\definecolor{promptcolorheader}{HTML}{bdbcec}
\definecolor{skillblue}{RGB}{220,235,255}

\newcommand{\ofit}[1]{\textcolor{red!70!black}{$-#1$}}
\newcommand{\gen}[1]{\textcolor{green!55!black}{$+#1$}}
\newcommand{\inc}[1]{%
  \rlap{\textsubscript{\scriptsize\color{green!60!black}#1}}%
}

\newcommand{\dec}[1]{%
  \rlap{\textsubscript{\scriptsize\color{red!70!black}#1}}%
}

\newcommand{\incb}[1]{\textsubscript{\scriptsize\color{green!60!black}#1}}

\newlength{\sbw}

\definecolor{subtabgray}{gray}{0.92}
\definecolor{promptcolor}{HTML}{E3F0FA}
\definecolor{promptcolorheader}{HTML}{B5D6ED}
\definecolor{prompttitletext}{HTML}{1B3A5C}
\lstdefinestyle{pythonstyle}{
    language=Python,
    basicstyle=\small\ttfamily,
    keywordstyle=\color{blue},
    stringstyle=\color{red},
    commentstyle=\color{green!50!black},
    morecomment=[l]{\#},
    showstringspaces=false,
    numbers=left,
    numberstyle=\tiny\color{gray},
    frame=single,
    breaklines=true,
    tabsize=4
}

\usepackage{xcolor} %
\usepackage{hyperref} %
\hypersetup{
    colorlinks=true,
}

\usepackage{xargs} %
\usepackage[colorinlistoftodos,prependcaption,textsize=tiny]{todonotes}
\newcommandx{\wrn}[2][1=]{\todo[linecolor=red,backgroundcolor=red!25,bordercolor=red,#1]{#2}}
\newcommandx{\cmt}[2][1=]{\todo[linecolor=blue,backgroundcolor=blue!25,bordercolor=blue,#1]{#2}}

\definecolor{EBBlue}{HTML}{DCEAF7}
\definecolor{EBOrange}{HTML}{FFE5BF}
\definecolor{ModelGray}{HTML}{F0F0F0}
\newcommand{\ebpair}[2]{%
    \colorbox{EBBlue}{\strut #1}\kern1pt%
    \colorbox{EBOrange}{\strut #2}%
}

\def \H1{H1}
\def \G1{G1}

\IEEEoverridecommandlockouts

\begin{document}
\makeatletter
\let\@oldmaketitle\@maketitle%

\makeatother

\title{
SkillBoost: A Constrained Exploration-Exploitation Process for Mitigating Skill Overfitting
}

\def\cameraready{0}  %

\author{
\IEEEauthorblockN{
Hongqiang Lin$^\dagger$\textsuperscript{1},
Chao Liu$^\dagger$\textsuperscript{2},
Xiaofan Bai\textsuperscript{2},
Xuan Jin\textsuperscript{2},
Yuhong Li\textsuperscript{2},
Nenggan Zheng$^*$\textsuperscript{1},
Xipeng Cao$^*$\textsuperscript{2}
}
\IEEEauthorblockA{
\textsuperscript{1}Zhejiang University\qquad
\textsuperscript{2}Alibaba Group
}
\thanks{$^\dagger$ Equal contribution. $^*$Corresponding authors.}
}

\maketitle

\begin{abstract}
Enabling large language model (LLM) agents to accumulate and reuse experience from past interactions remains a central challenge in real-world applications. A promising solution is to treat skills as trainable states and optimize them in the same way as model parameters in neural network training. However, data-driven skill optimization is prone to overfitting to the limited trajectories collected from real environments. Overexploiting these trajectories overfits the current batch, while unconstrained exploration causes regression on previously solved cases. This tension motivates a constrained search view of skill self-evolution, governed by an exploration--exploitation trade-off. We propose SkillBoost, a three-stage framework that mitigates both risks: structured exploitation localizes observed failures to editable skill components, prior-guided exploration draws on prior knowledge in the LLM to generate diverse repair candidates, and verified acceptance commits a candidate only when it improves performance within a regression bound. Experiments across 23 model--benchmark configurations show that SkillBoost achieves state-of-the-art performance while mitigating overfitting, outperforming both human-crafted and LLM-generated skills. Notably, it matches or even surpasses the performance gains typically achieved through post-training. Furthermore, transfer experiments further show that optimized skills can be reused by other agents on similar tasks. The code is available at \url{https://github.com/HQ-Lin/SkillBoost}.
\end{abstract}
\section{Introduction}
Large language model (LLM) agents are increasingly deployed in real-world scenarios where they must handle a sequence of tasks arriving over time \cite{opus45report,qwen37,qwen37plus,wang2025openhands}. In such settings, agents are expected not only to solve isolated tasks but also to accumulate experience and improve their behavior across repeated interactions \cite{fang2025comprehensive}. A straightforward way to adapt agents is to post-train their parameters. However, post-training often requires substantial computation, curated data, and validation, making rapid adaptation difficult under shifting data distributions. These limitations motivate skill self-evolution as a lightweight alternative: agents continuously refine reusable external skills based on historical interactions, enabling rapid and iterative adaptation without modifying the underlying LLM parameters \cite{jiang2026sok,xu2026agent,rao2026skillopt}.

\begin{figure}[t!] 
    \centering       
    \includegraphics[width=0.99\linewidth]{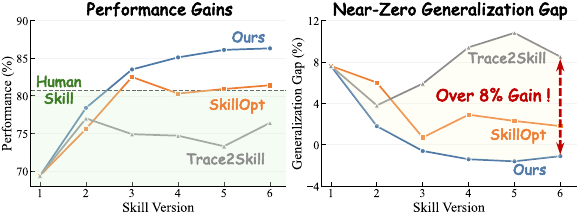}       
    \caption{Effective skill updates should improve performance without sacrificing generalization ability. Our method steadily improves performance across six skill versions, surpasses the human-crafted skill baseline, and keeps the generalization gap near zero.} 
    \label{fig:toy}
\end{figure}

Building on this lightweight adaptation paradigm, a prominent line of work casts skill self-evolution as an optimization problem. In these methods, skills are treated as external parameters that can be updated either through a process similar to gradient-based training \cite{yang2026skillopt,wang2026skillgrad,yu2026skilladaptor,liu2026skillforge} or, more recently, through zeroth-order search over skill text \cite{shen2026skillopt}. These methods share a common objective: they minimize the loss only on the current batch of interactions. Because each batch contains only a few trajectories, the learned skill may fit the observed cases well but fail to generalize to unseen tasks \cite{zheng2025lifelongagentbench,zhong2026skilllearnbench,ma2026skillclaw}. Fitting a skill too closely to such a narrow sample may therefore bias it toward case-specific patterns that are unlikely to appear again. Figure~\ref{fig:toy} shows the consequence of this objective in deployment benchmarks: successive skill updates improve performance on observed cases, but these gains fail to transfer to held-out tasks, resulting in a large generalization gap. We refer to this phenomenon as \textbf{skill overfitting}: the skill gains accuracy on short-term observations while losing robustness to later tasks from different distributions. Since this problem stems from an objective that focuses only on fitting observed interactions, simply collecting more trajectories under the same objective cannot address its root cause.

This limitation calls for a different formulation. We argue that skill self-evolution is not a data-fitting problem, but a constrained search problem shaped by an exploration--exploitation trade-off \cite{sutton2018reinforcement}. \textbf{Exploitation} uses the current failure trajectories to attribute errors to the skill components that need to be repaired. \textbf{Exploration} draws on prior knowledge in the LLM to propose multiple candidate edits within this localized repair space. Either side alone is risky: overexploiting the current batch leads to overfitting, while unconstrained edits can break cases that already work. This trade-off gives rise to a multi-step sequential decision problem: each accepted edit changes the skill used in later rounds and thus changes the failures observed next. Skill self-evolution therefore requires repeated decisions over which failures to exploit, which candidate edits to explore, and which verified edits to keep.

We formalize this sequential decision problem as a Markov Decision Process (MDP). In each round, the agent starts from the current skill and executes tasks to collect environment feedback. Based on this feedback, it proposes candidate edits, evaluates the edited skills, and decides whether to accept an edit as the next skill. These steps instantiate the MDP components: the current skill is the state, candidate edits are the actions, and evaluation feedback determines whether an edit is accepted. Under this formulation, we propose \textbf{SkillBoost}, which runs in three stages that balance the exploration--exploitation trade-off in skill self-evolution. \emph{Structured exploitation} attributes the observed failures to specific skill components, confining the subsequent edits to those components while leaving the rest of the skill unchanged. \emph{Prior-guided exploration} then uses the LLM to generate diverse repair candidates for those components, drawing on prior knowledge rather than batch-specific patterns. Finally, \emph{verified acceptance} commits a candidate only when it improves performance without excessive regression on previously solved cases. An accepted edit must therefore fix the current failures while preserving existing behavior, so it repairs the failure mode instead of memorizing the batch.

We further evaluate SkillBoost across both single-turn reasoning and multi-turn agentic tasks. As demonstrated in Table~\ref{tab:experiment}, SkillBoost achieves consistent improvements over both human-crafted and LLM-generated baselines.  We summarize our contributions as follows:
\begin{itemize}
    \item We identify overfitting as a central challenge in skill self-evolution and recast the process as a heuristic search governed by an exploration--exploitation trade-off.
    \item We propose \textbf{SkillBoost}, a skill self-evolution framework that mitigates overfitting by localizing failures to specific skill components, generating diverse candidate edits, and accepting only verified improvements. 
    \item Evaluated on 23 model–benchmark configurations, SkillBoost achieves state-of-the-art results while mitigating overfitting to a negligible test–train gap. Remarkably, its performance improvements are on par with, or even superior to, those of post-training methods.
    \item Extensive ablation studies demonstrate the contribution of each component in SkillBoost. Transfer experiments also show that the learned skills generalize well across different agents.
	\end{itemize}

\section{Related Work}
\subsection{Agentic Skills as Procedural Abstractions}
Agentic skills serve as reusable procedural abstractions that guide LLM behavior without parameter modification, positioning them between atomic tool invocation and high-level task orchestration \cite{huang2026rethinkingmemorymechanismsfoundation,li2026organizing,vishe2026skill,wei2026evomemorybenchmarkingllmagent}. Skills are naturally suited for file-based representation because they encapsulate task decomposition strategies, execution workflows, and decision heuristics in a modular format that supports retrieval, composition, and editing \cite{ouyang2026skilloslearningskillcuration,wang2025inducing,wang2025agent}. This design has been widely adopted across systems that organize experiences into skill libraries \cite{huang2026raw,wang2024voyager}, autonomously discover skills through interaction \cite{ni2026trace2skill,wang2026skillxautomaticallyconstructingskill,yang2026autoskill,zheng2025skillweaverwebagentsselfimprove}, or construct multi-file skill packages with verifier-guided refinement \cite{zhang2026coevoskills}. However, these approaches typically treat skill generation as one-time synthesis, without addressing how skills should be iteratively refined as agents accumulate more experience. 

\subsection{Evolution of Agentic Skills}
Skill self-evolution enables LLM agents to accumulate and refine reusable competencies through closed-loop interaction \cite{chen2026scaling,fang2025comprehensive}. Existing work can be broadly organized into extraction-based and refinement-based approaches. Extraction-based methods distill execution trajectories into structured skill packages, typically through prompt-based trajectory-to-skill synthesis \cite{ni2026trace2skill,qiu2026autorefine,shi2026search,yang2026autoskill}. These methods are useful for skill initialization, but do not specify how skills should keep improving after deployment. Refinement-based methods improve an existing skill iteratively, either through lifecycle agents deployed in closed-loop settings \cite{he2026evotest,zhang2026coevoskills,lin2026muse} or by training a dedicated skill generator with reinforcement learning \cite{ouyang2026skilloslearningskillcuration,vishe2026skill,xia2026skillrl}. More recently, optimization-driven methods treat skills as external trainable parameters and update them through gradient-like or zeroth-order procedures \cite{wang2026skillgrad,yang2026skillopt,yu2026skilladaptor,liu2026skillforge,shen2026skillopt}. While these methods systematically improve skills, they overfit to the observed trajectory batch by encoding batch-specific patterns rather than transferable knowledge. Consequently, gains on current tasks fail to generalize, and edits benefiting observed cases often degrade previously stable behavior. This constitutes a skill-level form of overfitting that existing methods do not explicitly control. SkillBoost addresses this problem by balancing failure-driven exploitation with prior-guided exploration: it localizes failures to skill components, generates diverse candidate repairs, and commits only verified improvements.

\section{Methodology}
\begin{figure*}[t!]
    \centering
    \includegraphics[width=0.99\textwidth]{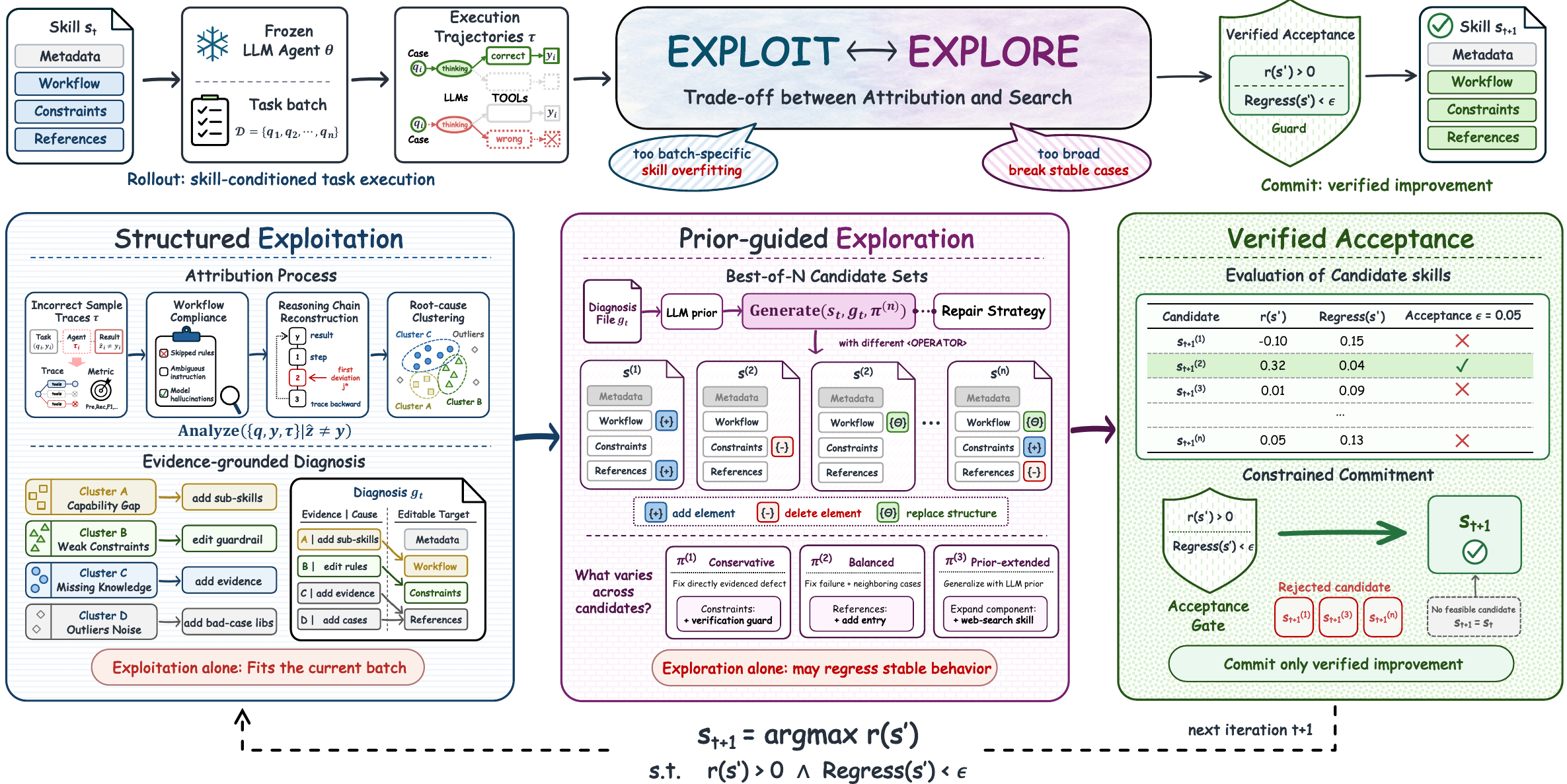}
    \caption{Overview of SkillBoost. \textbf{Top:} in each iteration, a frozen LLM agent executes tasks conditioned on the current skill $s_t$, and an acceptance gate commits only verified improvements to yield $s_{t+1}$. \textbf{Bottom:} backward optimization attributes failures to skill components (\emph{Structured Exploitation}), generates $N$ candidates under different repair strategies (\emph{Prior-guided Exploration}), and selects the highest-gain candidate passing the gate (\emph{Verified Acceptance}).}
	\label{fig:overview}
\end{figure*}

\subsection{Self-Evolution Formulation}
The current skill $s_t$ is represented as a structured state vector, whose coordinates correspond to key components (e.g., Metadata, Workflow, Constraints, References). Under this coordinate-based representation, the agent can identify and edit the components responsible for failures. The evolution process is naturally Markovian: the performance of the next version depends only on the current skill state and the applied edits. We therefore model self-evolution as an MDP $\mathcal{M}=(\mathcal{S}, \mathcal{A}, P, r)$. A state $s_t \in \mathcal{S}$ is the structured skill representation at version $t$. To evolve this state, the agent selects an action $a_t \in \mathcal{A}$ consisting of targeted edits to specific skill components, where each edit specifies an operation type, a target component, and the associated content. The transition function $P: \mathcal{S} \times \mathcal{A} \rightarrow \mathcal{S}$ deterministically applies these edits to yield $s_{t+1} = P(s_t, a_t)$. Finally, the reward $r$ measures the performance change caused by the modified components.
\subsection{Evolution Loop Overview}
We decompose each evolution iteration into a forward rollout phase and a backward optimization phase (see Figure~\ref{fig:overview}). The complete procedure is outlined in Algorithm~\ref{algo1}. 

\paragraph{Forward Rollout: Skill-Conditioned Task Execution.}
Given the current skill version $s_t$ at evolution step $t$, the LLM agent executes a batch of tasks from dataset $\mathcal{D} = \{(q_i, y_i)\}_{i=1}^{M}$, where $q_i$ is the task input and $y_i$ is the ground-truth answer. For each $q_i$, the agent produces a prediction $\hat{y}_i = (\tau_i, \hat{z}_i)$ conditioned on $s_t$: $(\tau_i, \hat{z}_i) \sim p_{\theta}(\cdot \mid q_i, s_t)$, where $\theta$ denotes the frozen LLM parameters, $\tau_i = \{(o_i^{(j)}, \texttt{tool}_i^{(j)}, \texttt{out}_i^{(j)}) \mid j=1, \cdots, m \}$ is the step-level execution trajectory consisting of observations, tool invocations, and tool outputs, and $\hat{z}_i$ is the final answer. The forward phase collects these trajectories for both successful and failed cases, providing causal evidence for fine-grained failure attribution in the backward optimization phase.

\paragraph{Backward Optimization: Feedback-Driven Skill Refinement.}
Backward optimization turns rollout feedback into targeted skill updates. Rather than fitting the whole skill to observed trajectories, it attributes failures to specific skill components, generates diverse candidate edits over those components, and accepts only verified improvements.

\subsection{Optimization in the Evolution Loop}

\paragraph{Structured Exploitation.}
Given failed cases $\{(q_i, y_i, \tau_i) \mid \hat{z}_i \neq y_i\}$ from forward rollout, the attribution stage analyzes the step-level trajectory $\tau_i = \{(o_i^{(j)}, \texttt{tool}_i^{(j)}, \texttt{out}_i^{(j)})\}$ to find where and why the current skill fails, where $o_i^{(j)}$, $\texttt{tool}_i^{(j)}$, and $\texttt{out}_i^{(j)}$ denote the observation, tool call, and output at step $j$. A preparation step extracts failed cases and summarizes their execution trajectories. An analyzer agent then runs three checks on $\tau_i$. First, \emph{Workflow compliance} checks whether the step sequence $\{(o_i^{(j)}, \texttt{tool}_i^{(j)})\}_{j}$ follows the workflow defined by the skill. It maps three trajectory patterns to three defects: steps that are skipped indicate weak procedural constraints; steps that follow the workflow yet lead to a wrong final answer indicate ambiguous instructions; and a conclusion that is not supported by the observations $o_i^{(j)}$ or tool outputs $\texttt{out}_i^{(j)}$ indicates that the skill does not enforce evidence-grounded reasoning. Second, \emph{reasoning chain reconstruction} walks $\tau_i$ backward from the final output to the rule each step uses, finds the first step $j^{\star}$ that departs from the intended rule, and checks whether that rule was bypassed or never reached. Third, \emph{root cause clustering} groups failures that share the same defect, so that a single edit can repair the whole group. It also splits root causes into two kinds: strategy defects, which can be fixed by editing the skill, and capability gaps, which recur across versions and cannot be fixed by editing rules. This stage uses only what the observed failures confirm. It produces a \emph{diagnosis} $g_t$ that is grounded in trajectory evidence:
\begin{equation}
\begin{aligned}
g_t &= \texttt{Analyze}(\{(q_i, y_i, \tau_i) \mid \hat{z}_i \neq y_i\}) \\
    &= \{(\text{cause}_k, \text{target}_k)\}_{k=1}^{K},
\end{aligned}
\end{equation}
where each pair $(\text{cause}_k, \text{target}_k)$ summarizes a group of failed cases whose trajectories exhibit the same defect. The $\text{cause}_k$ is grounded in step-level evidence from the trajectories in that group, citing the observations $o_i^{(j)}$, tool calls $\texttt{tool}_i^{(j)}$, and outputs $\texttt{out}_i^{(j)}$ up to the failure step $j^{\star}$. The $\text{target}_k$ specifies the editable skill component (e.g., workflow, constraints, or references) that should be repaired. Together, diagnosis $g_t$ maps trajectory-level failures to specific skill components, confining subsequent edits to the failure-attributed dimensions rather than the entire skill.

\paragraph{Prior-guided Exploration.}
The diagnosis $g_t$ identifies where the skill fails, but there are many valid ways to fix it, and the best one is not known in advance. A cautious edit changes only what the trajectories confirm: it rarely breaks working cases, but it may leave similar failures unfixed. A bolder edit uses the LLM prior to extend the fix to similar unseen failures: it covers more cases, but it may break cases that already passed. We therefore keep $g_t$ fixed and only change the repair strategy $\pi^{(n)}$. This strategy decides the edit scope and which error to fix first. Each strategy maps to a unique edit plan and produces one candidate:
\begin{equation}
a_t^{(n)} = \texttt{Generate}(s_t, g_t, \pi^{(n)}), \quad n = 1, \ldots, N,
\end{equation}
where each $a_t^{(n)}$ is a set of edits, and each edit adds, deletes, or replaces text at one place in the skill. Applying the transition yields the candidate skill $s_{t+1}^{(n)} = P(s_t, a_t^{(n)})$. All candidates share the same diagnosis $g_t$ and differ only in $\pi^{(n)}$, so each one is still based on a real failure but tries a different way to fix it. This $N$-candidate pool is our Best-of-$N$ exploration set: it lets the agent explore several prior-guided repair directions before acceptance.

\begin{figure}[t!]
    \centering
    \includegraphics[width=\linewidth]{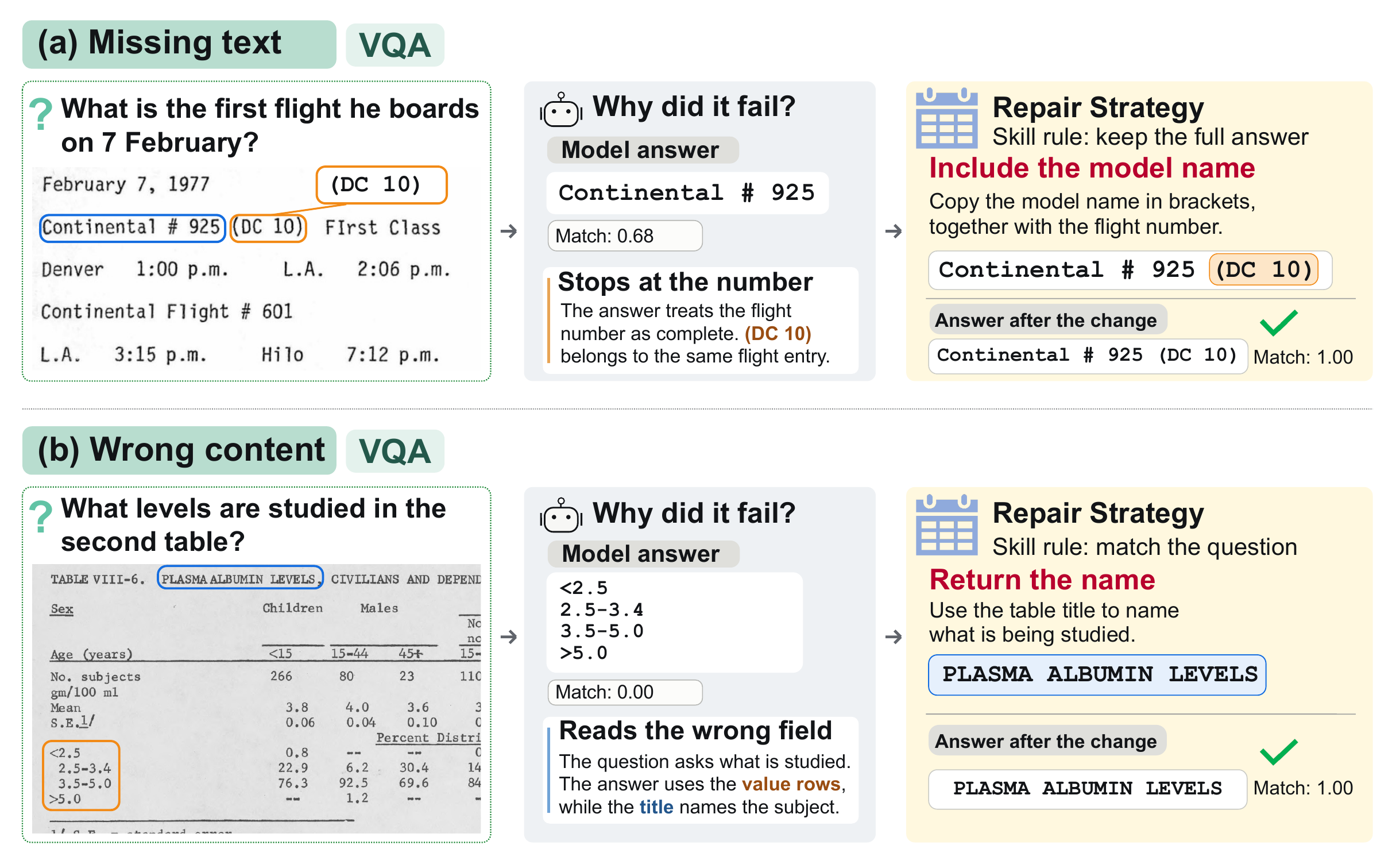}
    \caption{Visual failure attribution guides skill repair.}
    \label{fig:visual_attribution}
\end{figure}

\paragraph{Verified Acceptance.}
Given the Best-of-$N$ exploration set, verified acceptance commits only to candidates with verified net improvement. For a candidate skill $s' = P(s_t, a)$, where $P$ denotes the skill transition function, we define the gain $r(s') = \texttt{Eval}(s', \mathcal{D}) - \texttt{Eval}(s_t, \mathcal{D})$, and two case-level metrics: $\texttt{Fix}(s') = \frac{1}{M} \sum_i \mathbb{I}[\hat{z}_i^{(t)} \neq y_i \wedge \hat{z}_i' = y_i] $, the previously failed cases repaired by the edit, and $\texttt{Regress}(s') =\frac{1}{M} \sum_i \mathbb{I}[\hat{z}_i^{(t)} = y_i \wedge \hat{z}_i' \neq y_i] $, the previously correct cases broken by it. Since each case is scored as 0/1, the gain decomposes as $r(s') \propto \texttt{Fix}(s') - \texttt{Regress}(s')$. A candidate is accepted only if it improves the full-set score while keeping regressions under control: $\texttt{Accept}(s') \iff r(s') > 0 \land \texttt{Regress}(s') < \epsilon$, where $\epsilon$ is the regression threshold that limits how many previously solved cases an edit may break. The first condition is equivalent to $\texttt{Regress}(s') < \texttt{Fix}(s')$, so acceptance carries a clear meaning: a candidate must fix more cases than it breaks, and it is rejected once its regressions match or exceed its fixes, no matter how many failures it repairs. The final update selects the highest-gain candidate satisfying this gate:
\begin{equation}
    \label{eq:max_n}
    n^* = \arg\max_{n \in \{1,\ldots,N\}} r(s_{t+1}^{(n)})
    \quad \text{s.t.} \quad
    \texttt{Accept}(s_{t+1}^{(n)}).
\end{equation}
If no candidate passes the gate, the skill remains unchanged: $s_{t+1} = s_t$. This constrained selection ensures that the skill is updated only when a candidate brings a clear improvement.

\paragraph{Case Study.}
A repair strategy turns a failure diagnosis into a concrete edit by specifying which skill component to revise and how. Figure~\ref{fig:visual_attribution} illustrates this process with two examples from visual question answering. In the flight example, the model finds the correct flight number but leaves out the aircraft model, \texttt{(DC 10)}, printed next to it. The repair rule tells the model to copy both parts of the same entry. In the table example, the model copies numerical ranges, although the question asks what the table studies. The repair rule directs the model to the table title, which names \texttt{PLASMA ALBUMIN LEVELS}.

\begin{algorithm}[t!]
\caption{Procedure of SkillBoost.}
\label{algo1}
\begin{algorithmic}[1]
\REQUIRE Initial skill $s_0$, dataset $\mathcal{D}$, iterations $T$, pool size $N$.
\FOR{$t = 0, 1, \ldots, T-1$}
    \STATE \textbf{Forward Rollout:} Execute tasks conditioned on $s_t$: $\hat{y}_i \sim p_\theta(\cdot \mid q_i, s_t)$ for all $q_i \in \mathcal{D}$.
    \STATE $\triangleright$ \textit{Backward Optimization.}
    \STATE \textbf{Structured Exploitation:} Attribute failures to a diagnosis: $g_t \leftarrow \texttt{Analyze}(\{(q_i, y_i, \hat{y}_i) \mid \hat{z}_i \neq y_i\})$.
    \STATE \textbf{Prior-guided Exploration:} Form $N$ repair strategies $\pi^{(n)}$, $n=1,\ldots,N$, and generate candidates: $a_t^{(n)} \leftarrow \texttt{Generate}(s_t, g_t, \pi^{(n)})$, $\; s_{t+1}^{(n)} \leftarrow P(s_t, a_t^{(n)})$.
     \STATE \textbf{Verified Acceptance:} Solve optimization problem $n^* \leftarrow \arg\max_{n}\; r(s_{t+1}^{(n)}) \;\;\text{s.t.}\;\; \texttt{Accept}(s_{t+1}^{(n)}) $, then:
   \begin{equation*}
  s_{t+1} \leftarrow 
    \begin{cases} 
            s_{t+1}^{(n^*)} & \text{if } n^* \text{ exists}, \\ 
            s_t & \text{otherwise}.
    \end{cases}
   \end{equation*}
\ENDFOR
\RETURN Final skill $s_T$.
\end{algorithmic}
\end{algorithm}

\section{Experiments}

\makeatletter
\@ifundefined{scorew}{\newlength{\scorew}}{}
\@ifundefined{deltaw}{\newlength{\deltaw}}{}
\@ifundefined{modelw}{\newlength{\modelw}}{}
\@ifundefined{modelcolw}{\newlength{\modelcolw}}{}
\@ifundefined{skillw}{\newlength{\skillw}}{}
\makeatother

\settowidth{\scorew}{92.9}
\settowidth{\deltaw}{\incb{+47.4}}
\setlength{\modelw}{7.0em}
\setlength{\modelcolw}{1.4em}
\setlength{\skillw}{3.8em}

\providecommand{\scorecell}[2]{}
\renewcommand{\scorecell}[2]{%
  \makebox[\scorew][r]{#1}\makebox[\deltaw][l]{#2}%
}
\providecommand{\scorebcell}[2]{}
\renewcommand{\scorebcell}[2]{%
  \makebox[\scorew][r]{\textbf{#1}}\makebox[\deltaw][l]{#2}%
}
\providecommand{\modelleft}[1]{}
\renewcommand{\modelleft}[1]{%
  \multirow{6}{*}{\rotatebox{90}{\makebox[\modelw][c]{\small\textbf{#1}}}}%
}
\providecommand{\bhead}[1]{}
\renewcommand{\bhead}[1]{\textbf{#1}}

\makeatletter
\@ifundefined{scorew}{\newlength{\scorew}}{}
\@ifundefined{deltaw}{\newlength{\deltaw}}{}
\makeatother
\settowidth{\scorew}{92.9}
\settowidth{\deltaw}{\incb{+47.4}}

\providecommand{\scorecell}[2]{}
\renewcommand{\scorecell}[2]{%
  \makebox[\scorew][r]{#1}\makebox[\deltaw][l]{#2}%
}
\providecommand{\scorebcell}[2]{}
\renewcommand{\scorebcell}[2]{%
  \makebox[\scorew][r]{\textbf{#1}}\makebox[\deltaw][l]{#2}%
}

\providecommand{\modelrow}[1]{}
\renewcommand{\modelrow}[1]{%
  \multicolumn{5}{@{}c@{}}{%
    \colorbox{skillblue}{%
      \makebox[\dimexpr\columnwidth-2\fboxsep\relax][c]{\textbf{#1}}%
    }%
  }\\%
}
\begin{table}[t!]
    \centering
    \caption{Main results on held-out test splits. Green and red subscripts denote absolute changes relative to the no skill baseline within the same model block.}
    \label{tab:experiment}
    \setlength{\tabcolsep}{1.8pt}
    \renewcommand{\arraystretch}{1.08}
    \begin{tabularx}{\columnwidth}{@{}l
    >{\hsize=1.3\hsize\linewidth=\hsize\centering\arraybackslash}X
    *{3}{>{\hsize=.9\hsize\linewidth=\hsize\centering\arraybackslash}X}
    @{}}
    \toprule
    {\footnotesize\textbf{Skill Source}} &
    {\footnotesize\textbf{Spreadsheet}} &
    {\footnotesize\textbf{BFCL-v4}} &
    {\footnotesize\textbf{LiveMath}} &
    {\footnotesize\textbf{ALFWorld}} \\
    \midrule

        \modelrow{Claude-opus-4-6}
        No skill
            & \scorecell{50.0}{}
            & \scorecell{32.7}{}
            & \scorecell{28.6}{}
            & \scorecell{80.6}{} \\
        Human skill
            & \scorecell{72.5}{\inc{+22.5}}
            & \scorecell{26.5}{\dec{-6.2}}
            & \scorecell{37.8}{\inc{+9.2}}
            & \scorecell{82.2}{\inc{+1.6}} \\
        LLM skill
            & \scorecell{75.0}{\inc{+25.0}}
            & \scorecell{31.6}{\dec{-1.1}}
            & \scorecell{44.1}{\inc{+15.5}}
            & \scorecell{81.6}{\inc{+1.0}} \\
        Trace2Skill
            & \scorecell{73.1}{\inc{+23.1}}
            & \scorecell{42.3}{\inc{+9.6}}
            & \scorecell{46.4}{\inc{+17.8}}
            & \scorecell{85.3}{\inc{+4.7}} \\
        SkillOpt
            & \scorecell{70.4}{\inc{+20.4}}
            & \scorecell{45.7}{\inc{+13.0}}
            & \scorecell{49.6}{\inc{+21.0}}
            & \scorecell{89.8}{\inc{+9.2}} \\
        SkillBoost
            & \scorebcell{82.5}{\incb{\textbf{+32.5}}}
            & \scorebcell{54.0}{\incb{\textbf{+21.3}}}
            & \scorebcell{76.0}{\incb{\textbf{+47.4}}}
            & \scorebcell{92.9}{\incb{\textbf{+12.3}}} \\
        \midrule

        \modelrow{Qwen-3.7-max}
        No skill
            & \scorecell{54.6}{}
            & \scorecell{49.3}{}
            & \scorecell{19.2}{}
            & \scorecell{71.2}{} \\
        Human skill
            & \scorecell{62.5}{\inc{+7.9}}
            & \scorecell{51.0}{\inc{+1.7}}
            & \scorecell{26.4}{\inc{+7.2}}
            & \scorecell{73.4}{\inc{+2.2}} \\
        LLM skill
            & \scorecell{70.7}{\inc{+16.1}}
            & \scorecell{49.0}{\dec{-0.3}}
            & \scorecell{28.8}{\inc{+9.6}}
            & \scorecell{71.5}{\inc{+0.3}} \\
        Trace2Skill
            & \scorecell{67.1}{\inc{+12.5}}
            & \scorecell{38.0}{\dec{-11.3}}
            & \scorecell{33.6}{\inc{+14.4}}
            & \scorecell{74.9}{\inc{+3.7}} \\
        SkillOpt
            & \scorecell{73.9}{\inc{+19.3}}
            & \scorecell{32.7}{\dec{-16.6}}
            & \scorecell{32.8}{\inc{+13.6}}
            & \scorecell{75.4}{\inc{+4.2}} \\
        SkillBoost
            & \scorebcell{77.9}{\incb{\textbf{+23.3}}}
            & \scorebcell{52.6}{\incb{\textbf{+3.3}}}
            & \scorebcell{36.0}{\incb{\textbf{+16.8}}}
            & \scorebcell{82.0}{\incb{\textbf{+10.8}}} \\
        \midrule

        \modelrow{Qwen-3.6-plus}
        No skill
            & \scorecell{52.5}{}
            & \scorecell{50.7}{}
            & \scorecell{17.6}{}
            & \scorecell{60.6}{} \\
        Human skill
            & \scorecell{55.0}{\inc{+2.5}}
            & \scorecell{47.5}{\dec{-3.2}}
            & \scorecell{30.4}{\inc{+12.8}}
            & \scorecell{62.3}{\inc{+1.7}} \\
        LLM skill
            & \scorecell{60.0}{\inc{+7.5}}
            & \scorecell{51.7}{\inc{+1.0}}
            & \scorecell{25.6}{\inc{+8.0}}
            & \scorecell{56.9}{\dec{-3.7}} \\
        Trace2Skill
            & \scorecell{71.5}{\inc{+19.0}}
            & \scorecell{24.0}{\dec{-26.7}}
            & \scorecell{30.4}{\inc{+12.8}}
            & \scorecell{67.7}{\inc{+7.1}} \\
        SkillOpt
            & \scorecell{63.6}{\inc{+11.1}}
            & \scorecell{31.3}{\dec{-19.4}}
            & \scorecell{30.4}{\inc{+12.8}}
            & \scorecell{72.1}{\inc{+11.5}} \\
        SkillBoost
            & \scorebcell{75.0}{\incb{\textbf{+22.5}}}
            & \scorebcell{53.0}{\incb{\textbf{+2.3}}}
            & \scorebcell{34.4}{\incb{\textbf{+16.8}}}
            & \scorebcell{87.5}{\incb{\textbf{+26.9}}} \\
        \midrule

        \modelrow{DeepSeek-v4-pro}
        No skill
            & \scorecell{49.6}{}
            & \scorecell{27.3}{}
            & \scorecell{28.0}{}
            & \scorecell{55.2}{} \\
        Human skill
            & \scorecell{70.4}{\inc{+20.8}}
            & \scorecell{20.7}{\dec{-6.6}}
            & \scorecell{24.8}{\dec{-3.2}}
            & \scorecell{54.5}{\dec{-0.7}} \\
        LLM skill
            & \scorecell{70.0}{\inc{+20.4}}
            & \scorecell{22.0}{\dec{-5.3}}
            & \scorecell{29.6}{\inc{+1.6}}
            & \scorecell{57.1}{\inc{+1.9}} \\
        Trace2Skill
            & \scorecell{60.4}{\inc{+10.8}}
            & \scorecell{28.4}{\inc{+1.1}}
            & \scorecell{33.6}{\inc{+5.6}}
            & \scorecell{59.6}{\inc{+4.4}} \\
        SkillOpt
            & \scorecell{69.6}{\inc{+20.0}}
            & \scorecell{34.6}{\inc{+7.3}}
            & \scorecell{38.4}{\inc{+10.4}}
            & \scorecell{60.5}{\inc{+5.3}} \\
        SkillBoost
            & \scorebcell{80.0}{\incb{\textbf{+30.4}}}
            & \scorebcell{44.6}{\incb{\textbf{+17.3}}}
            & \scorebcell{46.4}{\incb{\textbf{+18.4}}}
            & \scorebcell{64.3}{\incb{\textbf{+9.1}}} \\
        \midrule

        \modelrow{Kimi-k2.6}
        No skill
            & \scorecell{45.0}{}
            & \scorecell{27.0}{}
            & \scorecell{40.0}{}
            & \scorecell{63.5}{} \\
        Human skill
            & \scorecell{50.7}{\inc{+5.7}}
            & \scorecell{24.3}{\dec{-2.7}}
            & \scorecell{42.4}{\inc{+2.4}}
            & \scorecell{63.0}{\dec{-0.5}} \\
        LLM skill
            & \scorecell{60.0}{\inc{+15.0}}
            & \scorecell{26.0}{\dec{-1.0}}
            & \scorecell{44.0}{\inc{+4.0}}
            & \scorecell{64.4}{\inc{+0.9}} \\
        Trace2Skill
            & \scorecell{46.8}{\inc{+1.8}}
            & \scorecell{28.0}{\inc{+1.0}}
            & \scorecell{46.4}{\inc{+6.4}}
            & \scorecell{64.7}{\inc{+1.2}} \\
        SkillOpt
            & \scorecell{60.4}{\inc{+15.4}}
            & \scorecell{29.3}{\inc{+2.3}}
            & \scorecell{48.8}{\inc{+8.8}}
            & \scorecell{68.1}{\inc{+4.6}} \\
        SkillBoost
            & \scorebcell{62.5}{\incb{\textbf{+17.5}}}
            & \scorebcell{34.7}{\incb{\textbf{+7.7}}}
            & \scorebcell{50.4}{\incb{\textbf{+10.4}}}
            & \scorebcell{70.1}{\incb{\textbf{+6.6}}} \\
        \bottomrule
    \end{tabularx}
\end{table}

\begin{table}[t!]
\centering
\caption{Performance comparison on EMMA and DocVQA.}
\label{tab:multimodal}
\renewcommand{\arraystretch}{1.15}
\setlength{\tabcolsep}{6pt}
\begin{adjustbox}{width=\columnwidth}
\begin{tabular}{lccc}
\toprule
\textbf{Skill Source}
& \textbf{Qwen-3.6-plus}
& \textbf{Qwen-3.7-plus}
& \textbf{Kimi-k2.6} \\
\midrule
\rowcolor{blue!8}
\multicolumn{4}{c}{\strut\textbf{DocVQA}} \\
\addlinespace[3pt]
No skill    & 86.0 & 86.7 & 86.7 \\
Human skill & 86.7\inc{+0.7}
            & 89.3\inc{+2.6}
            & 90.0\inc{+3.3} \\
LLM skill   & 90.0\inc{+4.0}
            & 88.7\inc{+2.0}
            & 90.7\inc{+4.0} \\
\rowcolor{gray!15}
\textbf{SkillBoost}
            & \textbf{92.0}\inc{\textbf{+6.0}}
            & \textbf{90.7}\inc{\textbf{+4.0}}
            & \textbf{91.3}\inc{\textbf{+4.6}} \\
\addlinespace[6pt]
\rowcolor{blue!8}
\multicolumn{4}{c}{\strut\textbf{EMMA}} \\
\addlinespace[3pt]
No skill    & 44.6 & 49.1 & 50.0 \\
Human skill & 50.0\inc{+5.4} & 54.6\inc{+5.5} & 56.4\inc{+6.4} \\
LLM skill   & 47.3\inc{+2.7} & 57.3\inc{+8.2} & 59.1\inc{+9.1} \\
\rowcolor{gray!15}
\textbf{SkillBoost} & \textbf{62.7}\inc{\textbf{+18.1}} & \textbf{70.0}\inc{\textbf{+20.9}} & \textbf{72.7}\inc{\textbf{+22.7}} \\
\bottomrule
\end{tabular}
\end{adjustbox}
\end{table}

\subsection{Experiment Setup}

\paragraph{Benchmarks.}
We evaluate Claude, Qwen, Kimi, and DeepSeek on held-out test splits from five agent benchmarks: SpreadsheetBench~\cite{ma2024spreadsheetbench} for multi-round code generation, BFCL-v4~\cite{patil2025berkeley} for multi-turn tool calling, LiveMathematicianBench~\cite{he2026livemathematicianbench} for mathematical reasoning, ALFWorld~\cite{shridhar2021alfworld} for embodied interaction, and DocVQA~\cite{Mathew2021docvqa} along with EMMA~\cite{hao2025can} for multimodal question answering. These benchmarks expose strong overfitting pressure because their training splits are orders of magnitude smaller than post-training corpora and production data.

\paragraph{Baselines.}
We compare SkillBoost with five baselines. \emph{No-Skill} uses the frozen target model with its default system prompt. \emph{Human Skill} uses expert-written task skills. \emph{One-Shot LLM Skill} is generated once from task descriptions without refinement. \emph{Trace2Skill} \cite{ni2026trace2skill} distills skills from trajectories. \emph{SkillOpt} \cite{yang2026skillopt} is a state-of-the-art skill optimization method. 

\subsection{Performance Comparison}
\paragraph{Main Results.}
Across 20 model--benchmark pairs (Table~\ref{tab:experiment}), SkillBoost matches or exceeds the strongest no-skill, human-skill, and LLM-skill baselines. Gains are largest for Claude-opus-4-6 on LiveMathematicianBench and SpreadsheetBench, where one-shot prompting struggles to enforce procedural rules. Human-written and LLM-generated skills sometimes degrade performance, especially on BFCL-v4 and ALFWorld, whereas SkillBoost consistently preserves or improves performance. This supports iterative optimization as more reliable than one-shot skill generation. On DocVQA and EMMA (Table~\ref{tab:multimodal}), SkillBoost outperforms human-written and LLM-generated skills on Qwen-3.6-plus, Qwen-3.7-plus, and Kimi-k2.6, extending its benefits to multimodal reasoning.

\paragraph{Overfitting Analysis.}
Table~\ref{tab:overfitting} reports the generalization gap $\Delta = \text{Test} - \text{Train}$ per benchmark and backbone. SkillOpt and Trace2Skill consistently suffer large negative $\Delta$ values, indicating overfitting from purely exploiting the current training split. In contrast, SkillBoost keeps $\Delta$ near zero or positive across benchmarks and backbones. This confirms that balancing failure-driven exploitation with prior-guided exploration prevents memorizing transient trajectories while preserving test generalization.

\begin{table}[t!]
    \centering
    \caption{Overfitting severity $\Delta = \text{Test} - \text{Train}$. Negative values (red) denote overfitting, non-negative values (green) indicate the absence of overfitting. SkillOpt and Trace2Skill exhibit large negative gaps, whereas SkillBoost remains near zero.}
    \label{tab:overfitting}
    \renewcommand{\arraystretch}{1.12}
    \setlength{\tabcolsep}{2pt}
    \begin{adjustbox}{width=\columnwidth}
    \begin{tabular}{ll cc|c}
        \toprule
        \textbf{Benchmark}
        & \textbf{Backbone}
        & \textbf{SkillOpt}
        & \textbf{Trace2Skill}
        & \textbf{SkillBoost} \\
        \midrule

        \multirow{5}{*}{Spreadsheet}
          & Claude-opus-4-6
          & \ofit{4.6} & \ofit{5.7} & \gen{1.3} \\
          & Qwen-3.7-max
          & \ofit{12.4} & \ofit{12.9} & \ofit{0.9} \\
          & Qwen-3.6-plus
          & \ofit{3.9} & \ofit{14.8} & \gen{2.5} \\
          & DeepSeek-v4-pro
          & \ofit{6.7} & \ofit{8.4} & \ofit{1.3} \\
          
          & Kimi-k2.6
          & \ofit{2.1} & \ofit{7.0} & \ofit{3.8} \\
        \midrule

        \multirow{5}{*}{BFCL-v4}
          & Claude-opus-4-6
          & \ofit{0.6} & \ofit{0.9} & \gen{0.4} \\

          & Qwen-3.7-max
          & \ofit{9.0} & \ofit{12.8} & \gen{0.1} \\

          & Qwen-3.6-plus
          & \ofit{7.0} & \ofit{17.7} & \ofit{0.3} \\

          & DeepSeek-v4-pro
          & \ofit{8.7} & \ofit{9.1} & \ofit{0.4} \\
          & Kimi-k2.6
          & \ofit{4.2} & \ofit{3.6} & \ofit{0.7} \\

        \midrule

        \multirow{5}{*}{LiveMath}
          & Claude-opus-4-6
          & \ofit{4.7} & \ofit{10.7} & \gen{1.7} \\

          & Qwen-3.7-max
          & \ofit{15.8} & \ofit{12.1} & \ofit{1.1} \\
          & Qwen-3.6-plus
          & \ofit{26.7} & \ofit{12.5} & \gen{0.1} \\
          & DeepSeek-v4-pro
          & \ofit{13.0} & \ofit{15.0} & \gen{0.7} \\
          & Kimi-k2.6
          & \ofit{19.8} & \ofit{7.9} & \ofit{1.0} \\
        \midrule

        \multirow{5}{*}{ALFWorld}
          & Claude-opus-4-6
          & \ofit{2.9} & \ofit{6.7} & \gen{1.6} \\
          & Qwen-3.7-max
          & \ofit{3.3} & \ofit{7.8} & \gen{0.7} \\
          & Qwen-3.6-plus
          & \ofit{1.9} & \ofit{8.3} & \gen{0.8} \\
          & DeepSeek-v4-pro
          & \ofit{5.5} & \ofit{11.1} & \ofit{1.0} \\
          & Kimi-k2.6
          & \ofit{2.6} & \ofit{5.3} & \ofit{1.2} \\

        \bottomrule
    \end{tabular}
    \end{adjustbox}
\end{table}

\subsection{Ablation Study}
\paragraph{Exploitation Module.}
SkillBoost's structured exploitation module combines workflow compliance verification, reasoning chain reconstruction, and root-cause clustering. Ablating any component reduces performance across all models and benchmarks (Table~\ref{tab:ablation_attribution_module}).The drop is especially clear on LiveMath, where exact-match scoring makes the final answer highly sensitive to reasoning errors. A wrong attribution can lead SkillBoost to edit the wrong step, causing the error to propagate through the chain. SpreadsheetBench and ALFWorld are less uniformly sensitive because their tasks provide more execution structure or environment feedback. The full model still performs best in every setting, showing that all three components are needed for reliable attribution.

\paragraph{Anti-Regression Behavior of Verified Acceptance.}
During skill self-evolution, \textit{regressions} occur when a new version fails previously solved cases. Verified acceptance rejects candidates whose regressions offset their repairs: $\texttt{Accept}(s')\iff r(s') > 0 \land \texttt{Regress}(s') < \epsilon$. On SpreadsheetBench (Kimi-k2.6), expanding a rule from 87 to 150 lines fixes 20 cases but breaks 23 (net $-3$), triggering rejection. In the w/o Gate setting (Table~\ref{tab:no_gate_ablation}), each round accepts the candidate that repairs the most cases on the failure set without full-set back-testing. This overlooks regressions and reduces accuracy across all models and benchmarks.

\subsection{Cross-Agent Skill Transfer}
We study whether skills optimized by SkillBoost transfer across agents and related benchmarks. Results are reported in Table~\ref{tab:cross-generalization}.

\paragraph{Cross-Model Generalization.}
Across four benchmarks and two target models, transferred skills outperform human-written skills in all settings. These results suggest that SkillBoost learns reusable skills that benefit agents beyond the source model, although the extent of transfer depends on the specific benchmark and target model.

\begin{table}[t!]
\centering
\small
\caption{Ablation of SkillBoost modules across three models and three benchmarks: Sheet (SpreadsheetBench), Math (LiveMath), and ALF (ALFWorld). Each cell reports test accuracy (\%) after removing one module, and the shaded row is the full SkillBoost. Lower means more important.}
\label{tab:ablation_attribution_module}
\setlength{\tabcolsep}{2.8pt}
\renewcommand{\arraystretch}{1.1}
\begin{adjustbox}{width=\columnwidth}
\begin{tabular}{l ccc ccc ccc}
\toprule
\multirow{2}{*}{\textbf{Module}} 
& \multicolumn{3}{c}{\textbf{DeepSeek-v4-pro}}
& \multicolumn{3}{c}{\textbf{Qwen-3.7-max}}
& \multicolumn{3}{c}{\textbf{Kimi-k2.6}} \\
\cmidrule(lr){2-4}\cmidrule(lr){5-7}\cmidrule(lr){8-10}
& Sheet & Math & ALF & Sheet & Math & ALF & Sheet & Math & ALF \\
\midrule
-Workflow        & 68.6 & 29.6 & 51.8 & 73.3 & 25.6 & 77.2 & 61.4 & 42.4 & 60.8 \\
-Reasoning  & 64.3 & 27.2 & 49.8 & 74.2 & 21.6 & 74.8 & 57.9 & 37.6 & 53.8 \\
-RootCause  & 66.8 & 34.4 & 51.8 & 71.2 & 27.2 & 70.0 & 61.4 & 40.8 & 63.1 \\
\midrule
\rowcolor{gray!15}
\textbf{SkillBoost (full)} & \textbf{80.0} & \textbf{46.4} & \textbf{64.3}
& \textbf{77.9} & \textbf{36.0} & \textbf{82.0}
& \textbf{62.5} & \textbf{50.4} & \textbf{70.1} \\
\bottomrule
\end{tabular}
\end{adjustbox}
\end{table}

\begin{table}[t!]
\centering
\small
\caption{Ablation on the verified acceptance gate. w/o Gate: the top candidate on the failure set is accepted without full set back-testing. Removing the gate drops accuracy on all models and benchmarks.}
\label{tab:no_gate_ablation}
\setlength{\tabcolsep}{2.8pt}
\renewcommand{\arraystretch}{1.1}
\begin{adjustbox}{width=\columnwidth}
\begin{tabular}{l cc cc cc}
\toprule
\multirow{2}{*}{\textbf{Setting}}
& \multicolumn{2}{c}{\textbf{Claude-opus-4-6}}
& \multicolumn{2}{c}{\textbf{Kimi-k2.6}}
& \multicolumn{2}{c}{\textbf{DeepSeek-v4-pro}} \\
\cmidrule(lr){2-3}\cmidrule(lr){4-5}\cmidrule(lr){6-7}
& BFCL & Sheet & BFCL & Sheet & BFCL & Sheet \\
\midrule
w/o Gate & 37.7 & 75.7 & 17.0 & 59.3 & 36.3 & 67.5 \\
\rowcolor{gray!15}
\textbf{w/ Gate (ours)} & \textbf{48.5} & \textbf{82.5} & \textbf{20.8} & \textbf{62.5} & \textbf{44.6} & \textbf{80.0} \\
\bottomrule
\end{tabular}
\end{adjustbox}
\end{table}

\paragraph{Cross-Benchmark Generalization.}
Transferring the LiveMath skill optimized with Claude-opus-4-6 to OlympiadBench improves both DeepSeek-v4-pro and Kimi-k2.6. The skill encodes guidance on problem parsing, strategy selection, derivation, verification, and answer formatting. These steps help organize the model's reasoning and reduce error propagation, facilitating transfer across these mathematical benchmarks.

\begin{table}[t!]
\centering
\small
\caption{Skill transfer results. (a) Cross-model transfer: skills optimized by SkillBoost on the source model (Claude-opus-4-6) are transferred to target models. Human Skill denotes each target model with its human-written seed skill. (b) Cross-benchmark transfer: a skill optimized by SkillBoost on LiveMath is transferred to OlympiadBench.}
\label{tab:cross-generalization}
\setlength{\tabcolsep}{2.5pt}
\renewcommand{\arraystretch}{1.1}
\begin{adjustbox}{width=\columnwidth}
\begin{tabular}{lll ccc}
\toprule
\rowcolor{subtabgray}
\multicolumn{5}{l}{\textbf{(a) Cross-model transfer}} \\
\midrule
\textbf{Benchmark} & \textbf{Source} & \textbf{Target} & \textbf{Human Skill} & \textbf{Transferred Skill}  \\
\midrule
\multirow{2}{*}{Spreadsheet} & \multirow{2}{*}{Claude-opus-4-6}
  & DeepSeek-v4-pro & 70.4 & 71.1\inc{+0.7}  \\
  & & Kimi-k2.6       & 50.7 & 53.5\inc{+2.8}  \\
\midrule
\multirow{2}{*}{LiveMath} & \multirow{2}{*}{Claude-opus-4-6}
  & DeepSeek-v4-pro & 24.8 & 31.2\inc{+6.4}  \\
  & & Kimi-k2.6       & 42.4 & 44.8\inc{+2.4} \\
\midrule
\multirow{2}{*}{BFCL-v4} & \multirow{2}{*}{Claude-opus-4-6}
  & DeepSeek-v4-pro & 20.7 & 35.0\inc{+14.3}  \\
  & & Kimi-k2.6       & 12.5 & 16.3\inc{+3.8}  \\
\midrule
\multirow{2}{*}{ALFWorld} & \multirow{2}{*}{Claude-opus-4-6}
  & DeepSeek-v4-pro & 54.5 & 55.5\inc{+1.0}  \\
  & & Kimi-k2.6       & 63.0 & 66.1\inc{+3.1}  \\
\midrule
\rowcolor{subtabgray}
\multicolumn{5}{l}{\textbf{(b) Cross-benchmark transfer}} \\
\midrule
\textbf{Source} & \textbf{Target} & \textbf{Model} & \textbf{Human Skill} & \textbf{Transferred Skill} \\
\midrule
\multirow{2}{*}{LiveMath} & \multirow{2}{*}{OlympiadBench}
  & DeepSeek-v4-pro & 44.0 & 50.0\inc{+6.0} \\
  & & Kimi-k2.6       & 46.0 & 60.0\inc{+14.0}  \\
\bottomrule
\end{tabular}
\end{adjustbox}
\end{table}

\begin{figure}[t!]
    \centering
    \includegraphics[width=0.482\textwidth]{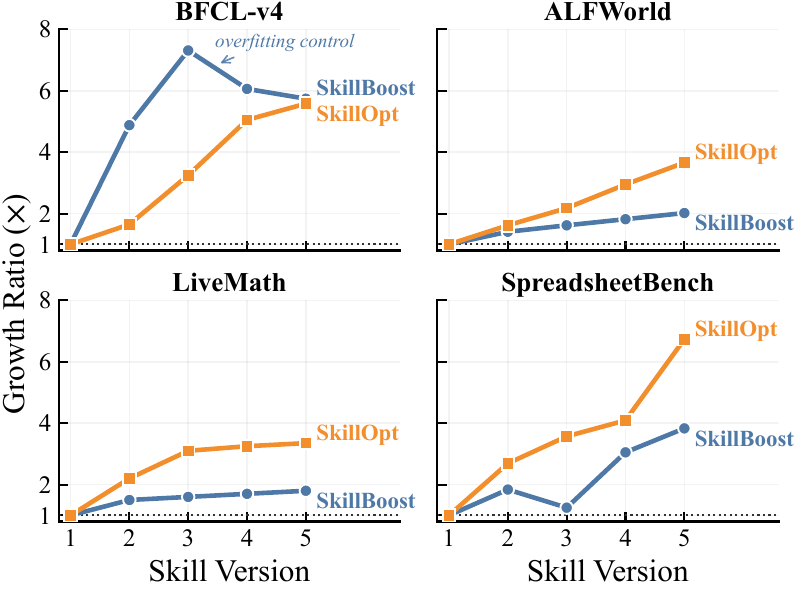}
    \caption{Skill-length growth across evolution versions. Each curve shows the token count of the evolved skill relative to its first version for SkillBoost and SkillOpt on four benchmarks.}
\label{fig:growth}
\end{figure}

\subsection{Empirical Analysis of Skill Evolution}
\paragraph{Skill Evolution Achieves Performance Comparable to Post-Training.}
We compare skill optimization with parameter updates on EB-ALFRED and EB-Habitat from EmbodiedBench \cite{yang2025embodiedbench}. We use Qwen3.7-plus as the backbone to generate skills, which are then applied to Qwen3-VL-4B, Qwen3-VL-8B, and Qwen3.5-9B during inference. Each target model is separately trained using supervised fine-tuning (SFT), GRPO \cite{guo2025deepseek}, and PPO \cite{schulman2017proximal} for comparison. SkillBoost outperforms GRPO in average score on both benchmarks for all three executors (Table~\ref{tab:skill_guidance}), and is competitive with SFT and PPO. These results support skill evolution as an effective alternative to parameter updates

\begin{table}[t]
    \centering
    \caption{Comparison of SkillBoost and SkillOpt with SFT, PPO, and GRPO on EB-ALFRED (blue) and EB-Habitat (orange). For each model and benchmark, the highest and second-highest Avg. scores are shown in \textbf{bold} and \underline{underlined}, respectively.}
    \label{tab:skill_guidance}
    \setlength{\tabcolsep}{5pt}
    \setlength{\fboxsep}{2pt}
    \renewcommand{\arraystretch}{1.15}
    \resizebox{0.486\textwidth}{!}{%
        \begin{tabular}{lccccccc}
        \toprule
        \multicolumn{8}{l}{
            \colorbox{EBBlue}{\textbf{EB-ALFRED}}
            \quad
            \colorbox{EBOrange}{\textbf{EB-Habitat}}
        } \\
        \textbf{Method / Model}
        & \textbf{Base} & \textbf{Com.} & \textbf{Comp.}
        & \textbf{Vis.} & \textbf{Spat.} & \textbf{Long} & \textbf{Avg.} \\
        \midrule
        \rowcolor{ModelGray}
        \multicolumn{8}{l}{\textit{Qwen3-VL-4B}} \\
        Base Model
        & \ebpair{18}{68} & \ebpair{14}{16} & \ebpair{26}{38} & \ebpair{14}{20} & \ebpair{14}{26} & \ebpair{26}{10} & \ebpair{18.7}{29.7} \\
        \quad + SkillBoost
        & \ebpair{30}{76} & \ebpair{22}{24} & \ebpair{36}{40} & \ebpair{24}{32} & \ebpair{24}{30} & \ebpair{28}{20} & \ebpair{27.3}{37.0} \\
        \quad + SkillOpt
        & \ebpair{26}{72} & \ebpair{18}{20} & \ebpair{32}{38} & \ebpair{18}{26} & \ebpair{22}{28} & \ebpair{24}{16} & \ebpair{23.3}{33.3} \\
        \quad + GRPO
        & \ebpair{24}{72} & \ebpair{18}{20} & \ebpair{30}{40} & \ebpair{18}{22} & \ebpair{18}{28} & \ebpair{32}{14} & \ebpair{23.3}{32.7} \\
        \quad + PPO
        & \ebpair{30}{78} & \ebpair{24}{28} & \ebpair{36}{46} & \ebpair{24}{32} & \ebpair{24}{32} & \ebpair{34}{18} & \ebpair{\underline{28.7}}{\underline{39.0}} \\
        \quad + SFT
        & \ebpair{34}{82} & \ebpair{28}{32} & \ebpair{40}{50} & \ebpair{28}{36} & \ebpair{28}{36} & \ebpair{38}{20} & \ebpair{\textbf{32.7}}{\textbf{42.7}} \\
        \midrule
        \rowcolor{ModelGray}
        \multicolumn{8}{l}{\textit{Qwen3-VL-8B}} \\
        Base Model
        & \ebpair{28}{58} & \ebpair{24}{20} & \ebpair{28}{20} & \ebpair{24}{18} & \ebpair{16}{18} & \ebpair{20}{10} & \ebpair{23.3}{24.0} \\
        \quad + SkillBoost
        & \ebpair{44}{68} & \ebpair{36}{38} & \ebpair{42}{38} & \ebpair{36}{34} & \ebpair{32}{32} & \ebpair{36}{22} & \ebpair{\underline{37.7}}{\textbf{38.7}} \\
        \quad + SkillOpt
        & \ebpair{36}{62} & \ebpair{28}{26} & \ebpair{36}{32} & \ebpair{30}{26} & \ebpair{26}{24} & \ebpair{28}{20} & \ebpair{30.7}{31.7} \\
        \quad + GRPO
        & \ebpair{34}{62} & \ebpair{28}{24} & \ebpair{32}{24} & \ebpair{28}{22} & \ebpair{22}{22} & \ebpair{24}{12} & \ebpair{28.0}{27.7} \\
        \quad + PPO
        & \ebpair{44}{72} & \ebpair{38}{34} & \ebpair{42}{34} & \ebpair{38}{32} & \ebpair{28}{32} & \ebpair{32}{22} & \ebpair{37.0}{\underline{37.7}} \\
        \quad + SFT
        & \ebpair{46}{70} & \ebpair{40}{30} & \ebpair{44}{32} & \ebpair{40}{28} & \ebpair{28}{28} & \ebpair{34}{20} & \ebpair{\textbf{38.7}}{34.7} \\
        \midrule
        \rowcolor{ModelGray}
        \multicolumn{8}{l}{\textit{Qwen3.5-9B}} \\
        Base Model
        & \ebpair{32}{70} & \ebpair{20}{28} & \ebpair{28}{30} & \ebpair{22}{24} & \ebpair{16}{18} & \ebpair{18}{14} & \ebpair{22.7}{30.7} \\
        \quad + SkillBoost
        & \ebpair{46}{78} & \ebpair{38}{40} & \ebpair{46}{46} & \ebpair{38}{40} & \ebpair{34}{36} & \ebpair{32}{32} & \ebpair{\textbf{39.0}}{45.3} \\
        \quad + SkillOpt
        & \ebpair{40}{72} & \ebpair{32}{34} & \ebpair{40}{40} & \ebpair{32}{34} & \ebpair{28}{30} & \ebpair{26}{24} & \ebpair{33.0}{39.0} \\
        \quad + GRPO
        & \ebpair{36}{76} & \ebpair{24}{32} & \ebpair{32}{36} & \ebpair{26}{30} & \ebpair{20}{22} & \ebpair{22}{18} & \ebpair{26.7}{35.7} \\
        \quad + PPO
        & \ebpair{48}{84} & \ebpair{36}{44} & \ebpair{44}{48} & \ebpair{38}{42} & \ebpair{30}{32} & \ebpair{36}{28} & \ebpair{\underline{38.7}}{\underline{46.3}} \\
        \quad + SFT
        & \ebpair{44}{86} & \ebpair{32}{44} & \ebpair{40}{50} & \ebpair{34}{42} & \ebpair{28}{34} & \ebpair{30}{28} & \ebpair{34.7}{\textbf{47.3}} \\
        \bottomrule
    \end{tabular}
    }
\end{table}

\paragraph{Skill Evolution Gains Increase with Model Size.}On ALFWorld, larger executor models derive greater utility from SkillBoost (Table~\ref{tab:skillboost_model_scale}). Their superior ability to apply guidance allows skill improvements to yield amplified downstream performance. This indicates that skill evolution and model scaling are complementary, with larger scales offering broader scope for benefiting from skill optimization.

\begin{table}[t!]
\centering
\caption{SkillBoost performance with different executor models.}
\label{tab:skillboost_model_scale}
\setlength{\tabcolsep}{5pt}
\renewcommand{\arraystretch}{1.1}
\begin{adjustbox}{max width=\columnwidth}
\begin{tabular}{lccc}
\toprule
\rowcolor{subtabgray}
\textbf{Executor} & \textbf{Baseline} & \textbf{Human Skill} & \textbf{SkillBoost} \\
\midrule
GPT oss 20B
  & 14.7 & 21.8 & \textbf{25.8}\inc{+4.1} \\
GPT oss 120B
  & 29.1 & 40.3 & \textbf{70.2}\inc{+29.9} \\
\midrule
DeepSeek-R1-Qwen-32B
  & 20.5 & 24.4 & \textbf{43.2}\inc{+18.8} \\
DeepSeek-R1-Llama-70B
  & 26.7 & 34.7 & \textbf{63.8}\inc{+29.1} \\
\bottomrule
\end{tabular}
\end{adjustbox}
\end{table}

\subsection{Discussion}
\paragraph{Skill Growth and Overfitting Control.}
Figure~\ref{fig:growth} tracks skill length during self-evolution. SkillOpt grows monotonically across all benchmarks; since each update absorbs rules fitted to training failures, unlimited accumulation makes the skill memorize case-specific patterns and overfit training traces. SkillBoost instead keeps skills compact, and on BFCL-v4 it actively prunes an over-grown skill, showing that its evolution loop acts as a regularizer that favors generalizable knowledge over raw accumulation.

\paragraph{Cost Analysis.}
We report two cost categories. The first is the \emph{optimization cost} of self-evolution. We control it with two-stage back-testing: all $N$ candidates are first screened on the previous-round failure set, and only the top-$2$ are evaluated on the full set. As a result, the cost grows linearly with $N$: $C(N)\approx N|\mathcal{F}|\tau_{\text{fail}}+2|\mathcal{D}|\tau_{\text{full}}=\beta N+\alpha$ (Figure~\ref{fig:bestofn}). SkillBoost and SkillOpt are comparable here, at 95K--154K tokens per round across benchmarks. The second is the \emph{evaluation cost} of running the optimized skill over the benchmark. This cost reaches 3--8M tokens per run and dominates the total inference cost (85--97\%). At deployment, SkillBoost reduces per-case inference tokens by 13.9\% on average compared with SkillOpt.

\paragraph{Effect of the Candidate Pool Size $N$.}
Figure~\ref{fig:bestofn} and Table~\ref{tab:bestofn} show how $N$ affects accuracy and token cost. Accuracy improves with $N$ but saturates quickly: increasing $N$ from 4 to 8 yields only 2.5 points on SpreadsheetBench and 1.1 points on ALFWorld, while token cost nearly doubles. We therefore set $N{=}4$ as the default, which captures most of the accuracy gain at moderate cost.

\begin{figure}[t!]
  \centering
    \includegraphics[width=0.995\linewidth]{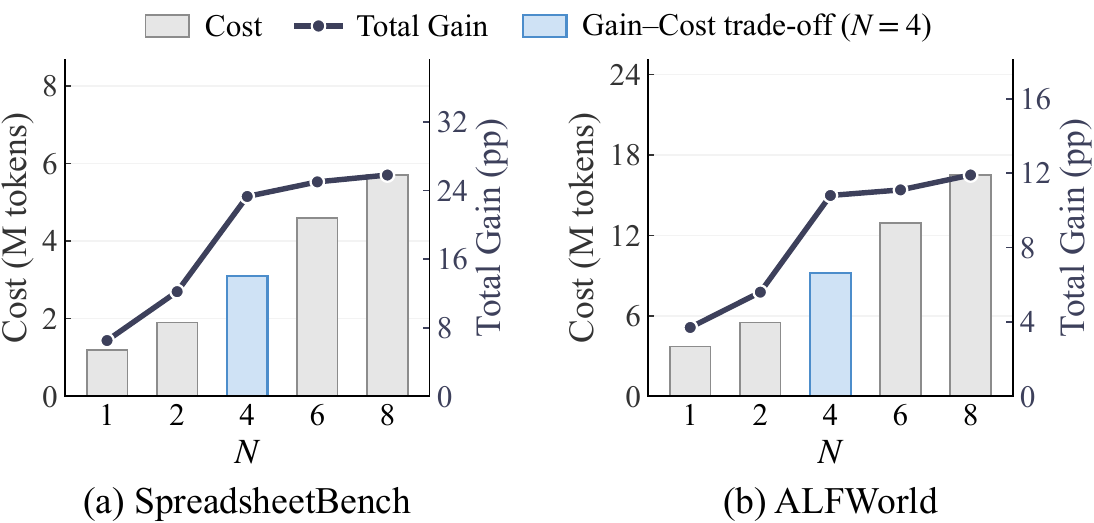}
    \caption{Best-of-$N$ results. Token Cost (bars, left axis) vs. Total Gain (line, right axis) as $N$ grows.}
    \label{fig:bestofn}
\end{figure}

\begin{table}[t!]
 \small
  \setlength{\tabcolsep}{8pt}
  \centering
  \caption{Best-of-$N$ results. Accuracy improves as the candidate size $N$ increases, but the gain saturates after $N=4$. The shaded column denotes our setting.}
  \label{tab:bestofn}
  \renewcommand{\arraystretch}{1.1}
  \begin{adjustbox}{width=\columnwidth}
  \begin{tabular}{l ccccc}
    \toprule
    \multirow{2}{*}{\textbf{Dataset}} & \multicolumn{5}{c}{\textbf{Final Acc. by $N$}} \\
    \cmidrule(lr){2-6}
     & $N{=}1$ & $N{=}2$ & \cellcolor{gray!15}$\mathbf{N{=}4}$ & $N{=}6$ & $N{=}8$ \\
    \midrule
    SpreadsheetBench & 61.1 & 66.8 & \cellcolor{gray!15}\textbf{77.9} & 79.6 & 80.4 \\
    ALFWorld         & 74.9 & 76.8 & \cellcolor{gray!15}\textbf{82.0} & 82.3 & 83.1 \\
    \bottomrule
  \end{tabular}
  \end{adjustbox}
\end{table}

\section{Conclusion}
In this work, we present \textbf{SkillBoost}, a skill self-evolution framework that mitigates overfitting by treating skill optimization as an exploration--exploitation trade-off rather than trajectory fitting. SkillBoost uses \emph{structured exploitation} to localize repair targets, \emph{prior-guided exploration} to generate diverse candidate edits, and \emph{verified acceptance} to commit only updates that improve performance without excessive regression. Across 23 model--benchmark configurations, SkillBoost outperforms human-crafted and LLM-generated skills, reduces overfitting, and produces skills that transfer well across agents. As with existing self-evolution paradigms, SkillBoost depends on the ability of the base model to attribute failures, generate meaningful edits, and evaluate candidates, so weaker models may limit the quality of self-evolution.

\section*{LLM Usage Statement}
We used OpenAI Codex to assist with grammar correction, language refinement, and code development. All AI-assisted content was reviewed by the authors, who take full responsibility for the correctness of the implementation, the validity of the experimental results, and the accuracy of the final manuscript.

\bibliographystyle{plain}
\bibliography{references}

\onecolumn
\newpage
\appendix
\setcounter{lemma}{0}
\setcounter{theorem}{0}

\section{Theoretical Analysis}
\label{app:theoretical_analysis}
The optimization problem in SkillBoost can be formalized as a constrained optimization problem. Let $G(x, x_t) = F(x_t) - F(x)$ and let $D(\cdot, \cdot)$ be a distance metric satisfying $D(x, x_t) > 0$ for all $x \neq x_t$. The update rule is:
\begin{equation}
\small
    x_{t+1} = \arg\min_{x} \; G(x, x_t) 	\quad \text{s.t.} \quad 
    \begin{cases}
        F(x_t) - F(x) < 0, \\
        D(x, x_t) \leq \epsilon,
    \end{cases}
\end{equation}
At each step, there are $N$ candidate points $\{x^{(1)}, \ldots, x^{(N)}\}$ to choose from within the $\epsilon$-ball around $x_t$. Let $x^*$ denote the target skill configuration and let $\ell=\|x^*-x_0\|$ be the distance from the initial skill to the target. We define the \emph{effective progress} at step $t$ as
\begin{equation}
    \delta_t = \max_{i=1,\ldots,N} \langle x^{(i)}-x_t, u_t\rangle,
\end{equation}
where $u_t = \frac{x^*-x_t}{\|x^*-x_t\|}$ is the unit vector pointing toward the target. In the continuous case, one can always move a distance $\epsilon$ along the target direction. With only $N$ discrete candidates, however, none may align perfectly with this direction. The following lemma quantifies the resulting loss in effective progress.

\begin{lemma}[Effective Step Size under Uniform Sampling]
\label{lem:effective_step}
Suppose the $N$ candidates are sampled independently and uniformly from $\mathcal{B}(x_t,\epsilon)\subset\mathbb{R}^d$. For large $d$ and moderate $N$, the expected effective progress satisfies
\begin{equation}
    \mathbb{E}[\delta_t]
    \approx
    \epsilon\sqrt{\frac{2\ln N}{d}}.
\end{equation}
\end{lemma}

\begin{proof}
For each candidate, define its projection onto the target direction as
\begin{equation}
    Z_i
    =
    \langle x^{(i)}-x_t,u_t\rangle .
\end{equation}
Since the candidates are sampled uniformly from $\mathcal{B}(x_t,\epsilon)$, the translated vectors $x^{(i)}-x_t$ are uniformly distributed over $\mathcal{B}(0,\epsilon)$. By symmetry and rotational invariance, the distribution of $Z_i$ does not depend on the particular direction $u_t$, and
\begin{equation}
    \mathbb{E}[Z_i]=0,
    \qquad
    \mathrm{Var}(Z_i)=\frac{\epsilon^2}{d+2}.
\end{equation}

For large $d$, the one-dimensional projection of a uniformly sampled point in the $d$-dimensional ball is well approximated by a Gaussian random variable. Since $\mathrm{Var}(Z_i)=\epsilon^2/(d+2)\to\epsilon^2/d$ as $d\to\infty$, we take
\begin{equation}
    Z_i
    \approx
    \mathcal{N}\left(0,\frac{\epsilon^2}{d}\right).
\end{equation}

The effective progress is the maximum projection among the $N$ candidates:
\begin{equation}
    \delta_t
    =
    \max_{i=1,\ldots,N} Z_i .
\end{equation}
By the classical extreme-value approximation for $N$ independent Gaussian variables with variance $\sigma^2$,
\begin{equation}
    \mathbb{E}\left[\max_{i=1,\ldots,N} Z_i\right]
    \approx
    \sigma\sqrt{2\ln N}.
\end{equation}
Substituting $\sigma=\epsilon/\sqrt{d}$ gives
\begin{equation}
    \mathbb{E}[\delta_t]
    \approx
    \epsilon\sqrt{\frac{2\ln N}{d}},
\end{equation}
which proves the claim.
\end{proof}

Lemma~\ref{lem:effective_step} shows that the effective step size scales as $\epsilon\sqrt{2\ln N/d}$, which is strictly less than the continuous optimum $\epsilon$ whenever $N < e^{d/2}$. A natural follow-up question is: given this reduced per-step progress, how many iterations are required to reach the target? The following theorem provides a lower bound.

\begin{theorem}[Iteration Complexity under Discrete Candidates]
\label{thm:main}
Under the uniform sampling assumption of Lemma~\ref{lem:effective_step}, suppose that candidates are independently resampled at each iteration. For a tolerance radius $\eta\geq 0$, define the first hitting time of the $\eta$-neighborhood of the target as
\begin{equation}
    T_\eta
    =
    \inf\{t\geq 0:\|x_t-x^*\|\leq \eta\}.
\end{equation}
Let
\begin{equation}
    \mu_N = \mathbb{E}[\delta_t]
\end{equation}
denote the expected effective progress under the sampling model. If $\mathbb{E}[T_\eta]<\infty$, then
\begin{equation}
    \mathbb{E}[T_\eta]
    \geq
    \frac{\ell-\eta}{\mu_N}.
\end{equation}
Using the approximation in Lemma~\ref{lem:effective_step}, this gives the approximate scaling
\begin{equation}
    \mathbb{E}[T_\eta]
    \approx
    \frac{\ell-\eta}{\epsilon}
    \sqrt{\frac{d}{2\ln N}}.
\end{equation}
\end{theorem}

\begin{proof}
Let
\begin{equation}
    r_t=\|x^*-x_t\|,
    \qquad
    v_t=x_{t+1}-x_t .
\end{equation}
For all $t<T_\eta$, we have $r_t>\eta$, and in particular $r_t>0$ when $\eta\geq 0$ and $x_t\neq x^*$. Hence
\begin{equation}
    u_t
    =
    \frac{x^*-x_t}{\|x^*-x_t\|}.
\end{equation}
By the Cauchy--Schwarz inequality (equivalently, $\|w\|\geq\langle w,\hat u\rangle$ for any unit vector $\hat u$),
\begin{equation}
    r_{t+1}
    =
    \|x^*-x_t-v_t\|
    \geq
    \langle x^*-x_t-v_t,\,u_t\rangle
    =
    r_t-\langle v_t,u_t\rangle .
\end{equation}
Therefore, the decrease in distance to the target is upper bounded by the progress along the target direction:
\begin{equation}
    r_t-r_{t+1}
    \leq
    \langle v_t,u_t\rangle .
\end{equation}
Since $x_{t+1}$ is selected from the sampled candidate set and $\delta_t$ is the maximum projection among all sampled candidates,
\begin{equation}
    \langle v_t,u_t\rangle
    \leq
    \delta_t .
\end{equation}
Combining the two inequalities gives
\begin{equation}
    r_t-r_{t+1}
    \leq
    \delta_t .
\end{equation}

By the definition of $T_\eta$, the process first enters the $\eta$-neighborhood of the target at time $T_\eta$, so $r_{T_\eta}\leq \eta$ and $r_0=\ell$. Summing the previous inequality over $t=0,\ldots,T_\eta-1$ yields
\begin{equation}
    \ell-\eta
    \leq
    r_0-r_{T_\eta}
    =
    \sum_{t=0}^{T_\eta-1}(r_t-r_{t+1})
    \leq
    \sum_{t=0}^{T_\eta-1}\delta_t .
\end{equation}

Let $\mathcal{F}_t$ denote the history before sampling the candidates at iteration $t$. Under independent resampling and rotational invariance of the uniform distribution in the $\epsilon$-ball, Lemma~\ref{lem:effective_step} implies that
\begin{equation}
    \mathbb{E}[\delta_t\mid\mathcal{F}_t]
    =
    \mu_N
    \approx
    \epsilon\sqrt{\frac{2\ln N}{d}}.
\end{equation}
Using the tower property and the stopping-time property of $T_\eta$, we have
\begin{equation}
    \mathbb{E}\left[
        \sum_{t=0}^{T_\eta-1}\delta_t
    \right]
    =
    \mathbb{E}\left[
        \sum_{t=0}^{T_\eta-1}
        \mathbb{E}[\delta_t\mid\mathcal{F}_t]
    \right]
    =
    \mu_N\mathbb{E}[T_\eta].
\end{equation}
Taking expectations in the pathwise bound above, we obtain
\begin{equation}
    \ell-\eta
    \leq
    \mu_N\mathbb{E}[T_\eta].
\end{equation}
Rearranging gives the exact bound in terms of $\mu_N$:
\begin{equation}
    \mathbb{E}[T_\eta]
    \geq
    \frac{\ell-\eta}{\mu_N}.
\end{equation}
Using the approximation in Lemma~\ref{lem:effective_step}, the lower bound scales as
\begin{equation}
    \frac{\ell-\eta}{\epsilon}
    \sqrt{\frac{d}{2\ln N}}.
\end{equation}
\end{proof}
\paragraph{Remarks.}
Theorem~\ref{thm:main} provides an expected lower-bound scaling for reaching an $\eta$-neighborhood of the target. Under the idealized uniform sampling model, the number of iterations scales at least as $\frac{\ell-\eta}{\epsilon}\sqrt{\frac{d}{2\ln N}}$. Compared to the continuous case, where the corresponding scale is $(\ell-\eta)/\epsilon$, the discrete setting introduces an additional $\sqrt{d/(2\ln N)}$ factor. This factor reflects the cost of selecting from a finite candidate pool that cannot fully cover all directions in the $\epsilon$-ball. Increasing $N$ improves the effective progress only logarithmically: doubling the candidate pool changes the lower-bound factor by $\sqrt{\frac{\ln N}{\ln(2N)}}$ , so the marginal gain diminishes as $N$ grows. Real candidates are also correlated rather than independent, so the true gain saturates no later than this bound predicts.

\section{Experiments}
\label{app:exp}
\subsection{Hyperparameter Settings}
\label{app:hyper}

Table~\ref{tab:hyperparameters} summarizes all hyperparameters used in this work, grouped by the component they belong to: LLM decoding, agent execution, Best-of-$N$ candidate selection, the skill evolution loop, and evaluation infrastructure.

\begin{table}[t]
\centering
\small
\caption{Hyperparameters used throughout this work.}
\label{tab:hyperparameters}
\setlength{\tabcolsep}{4pt}
\renewcommand{\arraystretch}{1.2}
\begin{adjustbox}{max width=\columnwidth}
\begin{tabular}{@{}lll@{}}
\toprule
\textbf{Hyperparameter} & \textbf{Value} & \textbf{Notes} \\
\midrule
\multicolumn{3}{@{}l}{\textit{LLM decoding}} \\
Temperature & 0.1 & \\
Max output tokens & 16{,}384  \\
Thinking mode & disabled \\
API retries & 5 &  \\
Random seed & 42 & Environment seed \\
API timeout & 120 s & Per-request (connect 30 s) \\
\midrule
\multicolumn{3}{@{}l}{\textit{Agent execution}} \\
Max ReAct iterations & 10 & Agent reasoning-action loop \\
Max environment steps & 50 & ALFWorld episode budget \\
Max codegen turns & 6 & SpreadsheetBench code-generation dialogue \\
Code execution timeout & 120 s & Per case (600 s per task overall) \\
\midrule
\multicolumn{3}{@{}l}{\textit{Best-of-$N$ selection}} \\
$N$ (candidates) & 4 (default); 2/6/8 & Ablation over candidate population size \\
Top-$K$ & 2 & Candidates promoted from Phase A to full evaluation \\
Selection metric & accuracy / success rate & Task-dependent primary metric \\
\midrule
\multicolumn{3}{@{}l}{\textit{Skill evolution loop}} \\
Guard samples & 10--20 & Random correct cases added against regressions \\
Diff budget & 400 lines/file & Max diff rendered into a repair strategy \\
\midrule
\multicolumn{3}{@{}l}{\textit{Evaluation infrastructure}} \\
Concurrency & 4--50 & Adaptively adjust based on retry conditions \\
Best-of-$N$ concurrency & 30 & Orchestrator-level parallel evaluations \\
\bottomrule
\end{tabular}
\end{adjustbox}
\end{table}

\subsection{Recommended Local Machine Setup}
\label{app:machine-setup}

All LLM inference in this work is performed through cloud APIs, so \emph{no local GPU is required}. The local machine only runs the evaluation harness: concurrent asynchronous API calls (up to 50 in flight), lightweight task environments (ALFWorld text environment, spreadsheet execution via \texttt{openpyxl}), the Docker evaluation harness, and analysis scripts. Table~\ref{tab:machine_setup} lists the recommended configuration.

\begin{table}[h!]
\centering
\small
\caption{Recommended local machine setup.}
\label{tab:machine_setup}
\setlength{\tabcolsep}{4pt}
\renewcommand{\arraystretch}{1.2}
\begin{adjustbox}{max width=\columnwidth}
\begin{tabular}{@{}lll@{}}
\toprule
\textbf{Component} & \textbf{Minimum} & \textbf{Recommended / Notes} \\
\midrule
CPU & 4 cores & 8+ cores; evaluation is I/O-bound async HTTP with up to 50 concurrent requests \\
GPU & none & Not required; all model inference is served by cloud APIs \\
RAM & 8 GB & 16--32 GB; upper end needed for Docker containers and parallel ALFWorld workers \\
Disk & 50 GB free & 200+ GB SSD if running the Docker harness (per-instance images); traces/logs grow to tens of GB over long evolution runs \\
Network & stable broadband & Low-latency, high-availability link to API endpoints; sustained 50-way concurrent HTTPS \\
OS & Linux / macOS & Any platform with Python 3.10+ and POSIX shell; scripts developed on macOS \\
Python & 3.10+ & Key packages: \texttt{httpx}, \texttt{agentscope}, \texttt{openpyxl}, \texttt{alfworld}, \texttt{umap-learn}, \texttt{matplotlib} \\
\bottomrule
\end{tabular}
\end{adjustbox}
\end{table}

\subsection{Experimental Setup for EmbodiedBench}
\label{app:exp_set}
EmbodiedBench evaluates whether multimodal models can complete language-specified tasks through visual perception, reasoning, and interaction with simulated environments. Its four environments cover high-level household planning and low-level navigation and manipulation. We focus on EB-ALFRED and EB-Habitat, where agents select high-level actions using instructions, egocentric observations, and interaction feedback. Task success requires the executed action sequence to achieve the specified goal. These tasks assess whether agents can ground instructions in the scene, order actions coherently, and adapt their decisions to changes in the environment.

EB-ALFRED builds on the ALFRED dataset and AI2-THOR simulator to evaluate household task decomposition through object search and manipulation. EB-Habitat uses the Language Rearrangement benchmark and Habitat 2.0 simulator to evaluate object rearrangement through navigation, picking, placing, and opening or closing receptacles. Navigation is restricted to receptacles, so agents may need to search multiple locations before finding a target object. Each benchmark contains 300 test instances, with 50 in each of six subsets. Base evaluates routine task planning, while Common Sense tests interpretation of indirect object descriptions using everyday knowledge. Complex Instruction assesses goal identification within longer, potentially distracting contexts. Visual Appearance and Spatial Awareness test object identification through visual attributes and relative positions, respectively. Long Horizon evaluates planning over extended action sequences. Together, these subsets distinguish difficulties in instruction understanding, object grounding, and sequential planning.

Skill self-evolution can improve performance in this setting by converting execution failures into reusable procedural guidance. The high-level action structure and environment feedback provide evidence for diagnosing errors in task decomposition and action ordering.

We use Qwen3.7-Plus as the backbone of SkillBoost to generate and optimize skills, which are applied to Qwen3-VL-4B, Qwen3-VL-8B, and Qwen3.5-9B at inference time. For the parameter-update baselines, each target model is trained separately. For PPO, each target model serves as the policy, with Qwen2.5-VL-3B as the critic. The critic estimates expected returns to support advantage estimation and guide policy updates. For SFT, Qwen3.7-plus serves as the teacher, generating reasoning trajectories for supervised fine-tuning of each target model. SFT trains the target models to reproduce the teacher's task-solving behavior through parameter updates.

\subsection{Further evidence of skill gains scaling with model size}
\label{app:scal}

\begin{table}[t!]
\centering
\caption{SkillBoost performance with different executor models on SpreadsheetBench. }
\label{tab:skillboost_spreadsheet_model_scale}
\setlength{\tabcolsep}{5pt}
\renewcommand{\arraystretch}{1.1}
\begin{adjustbox}{max width=\columnwidth}
\begin{tabular}{lccc}
\toprule
\rowcolor{subtabgray}
\multicolumn{4}{l}{\textbf{Skill Self-Evolution Across Executor Models}} \\
\midrule
\textbf{Executor} & \textbf{Baseline} & \textbf{Human Skill} & \textbf{SkillBoost} \\
\midrule
GPT oss 20B
  & 28.8 & 33.8 & \textbf{45.0}\inc{+16.2} \\
GPT oss 120B
  & 32.5 & 43.8 & \textbf{58.8}\inc{+26.3} \\
\midrule
DeepSeek-R1-Qwen-32B
  & 13.8 & 22.5 & \textbf{33.8}\inc{+20.0} \\
DeepSeek-R1-Llama-70B
  & 23.8 & 21.3 & \textbf{50.0}\inc{+26.2} \\
\bottomrule
\end{tabular}
\end{adjustbox}
\end{table}
On SpreadsheetBench, larger executors obtain greater gains from SkillBoost within both model groups (Table~\ref{tab:skillboost_spreadsheet_model_scale}). This trend suggests that executor capacity helps determine how effectively optimized procedural guidance translates into task performance, allowing skill evolution to complement model scaling. SkillBoost also improves performance where human skill degrades the baseline, suggesting that iterative refinement can adapt guidance to the executor's behavior and overcome limitations of static skills.

\subsection{Post-training Details}
\label{app:training-details}

\begin{table}[t!]
\centering
\caption{PPO Training Hyperparameters.}
\label{tab:app-training-hyperparameters}
\small
\setlength{\tabcolsep}{6pt}
\renewcommand{\arraystretch}{1.1}
\begin{tabular}{@{}ll@{\hspace{2em}}ll@{}} 
\toprule
\textbf{Hyperparameter} & \textbf{Value} & \textbf{Hyperparameter} & \textbf{Value} \\
\midrule
\multicolumn{4}{@{}l}{\textit{Model and Optimizer}} \\
Critic backbone & Qwen2.5-VL-3B-Instruct & Optimizer & AdamW \\
Actor LR & $1\times10^{-6}$ & Critic LR & $1\times10^{-5}$ \\
Weight decay & $0.01$ & & \\
\addlinespace
\multicolumn{4}{@{}l}{\textit{Advantage and PPO}} \\
Estimator & GAE ($\gamma=0.99, \lambda=0.95$) & Mini-batch size & $8$ \\
Epochs & $1$ & Entropy bonus & $0.001$ \\
Actor clip & $0.2$ & Critic clip & $0.5$ \\
\addlinespace
\multicolumn{4}{@{}l}{\textit{Rollout Collection}} \\
Parallel envs & $16$ & Max steps & $30$ \\
\bottomrule
\end{tabular}
\end{table}

\paragraph{PPO Configuration.}
During the experiments, both the actor and critic networks are optimized with AdamW using separate learning rates, and the detailed hyperparameters are listed in Table~\ref{tab:app-training-hyperparameters}. We collect rollouts from 16 parallel environments, where each episode is limited to 30 steps. The PPO mini-batch size is set to 8, and we perform one optimization epoch per rollout batch. We also examined KL regularization with coefficients of $0.01$ and $0.001$, but neither setting led to noticeable performance improvement.

\begin{table}[t!]
\centering
\caption{SFT Training Hyperparameters.}
\label{tab:app-sft-hyperparameters}
\small
\renewcommand{\arraystretch}{1.2}
\begin{tabular}{@{}lr@{}}
\toprule[1.5pt]
\textbf{Hyperparameter} & \textbf{Value} \\
\midrule[1pt]
\multicolumn{2}{@{}l}{\textit{Model and Data}} \\
Teacher model (default) & Qwen3.7-Plus \\
Training epochs & $5$ \\
Global effective batch size & $64$ \\
\addlinespace[0.3em]
\multicolumn{2}{@{}l}{\textit{Optimizer and Learning Rate}} \\
Optimizer & AdamW \\
Learning rate & $1\times10^{-5}$ \\
AdamW $\beta_1$ & $0.9$ \\
AdamW $\beta_2$ & $0.999$ \\
AdamW $\epsilon$ & $1\times10^{-8}$ \\
Weight decay & $0.01$ \\
Learning-rate schedule & Cosine \\
Warmup ratio & $0.1$ \\
\bottomrule[1.5pt]
\end{tabular}
\end{table}
\paragraph{SFT Procedure.}
Our Supervised Fine-Tuning (SFT) dataset is constructed using Qwen3.7-Plus as the teacher model. To ensure temporal consistency the visual input for each training instance corresponds to the observation available immediately prior to action execution. The teacher model generates the target response which incorporates the ground-truth action thereby explicitly aligning the decision-making process with the available visual evidence. The model is trained for five epochs using the AdamW optimizer coupled with a cosine learning-rate scheduler and a warmup ratio of $0.1$. We set the learning rate to $1\times10^{-5}$ and maintain a global effective batch size of 64. Comprehensive SFT hyperparameters are listed in Table~\ref{tab:app-sft-hyperparameters}.

\paragraph{GRPO Setup.}
For the Group Relative Policy Optimization (GRPO) baseline we set the group size to 8. Distinct from the actor-critic architecture this approach operates without a value network and computes group-relative advantages solely based on terminal rewards.

We recommend a minimum computational resource of 16 H20 GPUs.

\subsection{Empirical Analysis}
\label{app:analysis}
\begin{figure}[t!]
    \centering
    \includegraphics[width=0.49\textwidth]{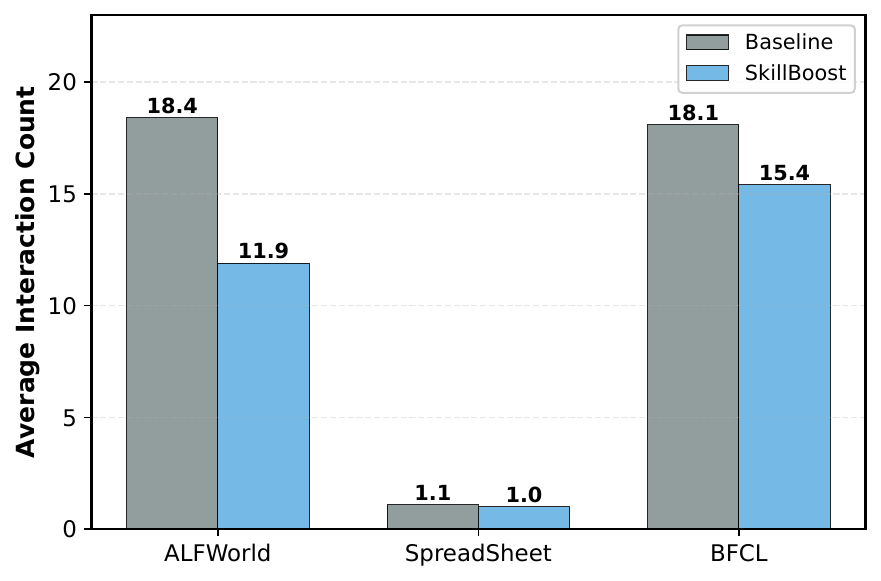}
    \caption{SkillBoost efficiency analysis across benchmarks.}
    \label{fig:interaction_count}
\end{figure}
\paragraph{Skill Enhances Reasoning Efficiency.}
Empirical results across three multi-turn reasoning benchmarks show that skills improve both efficiency and task success (Figure~\ref{fig:interaction_count}). In ALFWorld, skill guidance reduces average interaction steps by 35.3\% and wall-clock evaluation time by 57.2\%, mainly by reducing redundant exploration. In BFCL, skills reduce average function calls by 14.9\% and over-exploration cases ($>20$ calls) by 23.1\%, suggesting more disciplined tool selection. SpreadsheetBench shows a different pattern: because its loop terminates once code executes, even when the result is semantically wrong, skills mainly improve first-round code quality (+32.5 hard accuracy) rather than reducing turns (1.1$\to$1.0). Overall, skills yield the largest efficiency gains in sequential decision-making tasks with large action spaces, while in generation-heavy tasks they primarily improve output correctness.

\begin{figure*}[t!]
    \centering
    \begin{subfigure}{0.49\textwidth}
        \centering
        \includegraphics[width=\textwidth]{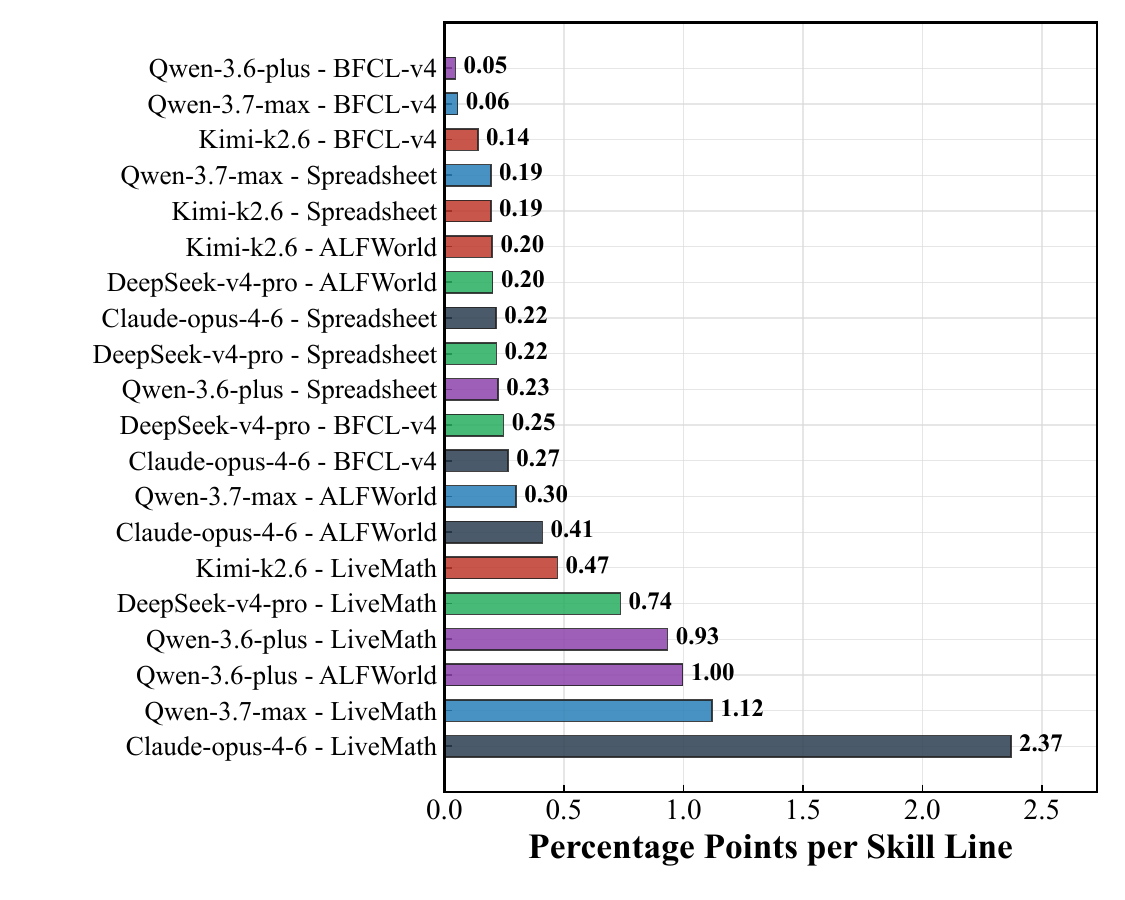}
        \caption{Skill line efficiency across model--benchmark settings.}
        \label{fig:skill_performance}
    \end{subfigure}
    \hfill
    \begin{subfigure}{0.49\textwidth}
        \centering
        \includegraphics[width=\textwidth]{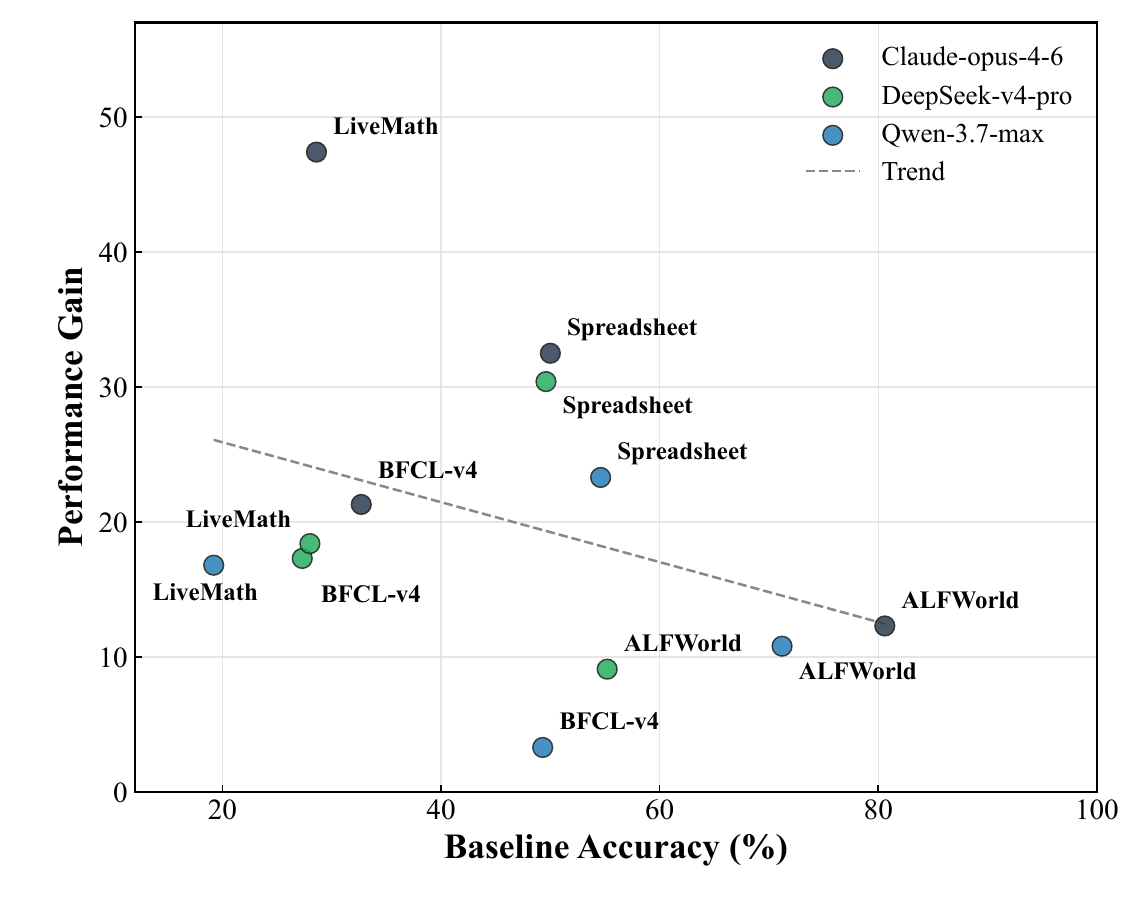}
        \caption{Performance gains over baseline accuracy.}
        \label{fig:efficiency}
    \end{subfigure}
    \caption{(a) Ranking of skill line efficiency measured in percentage points gained per additional skill line, showing substantial variation across tasks. (b) Relationship between model baseline accuracy and achievable performance gains, with trend line indicating the negative correlation.}
    \label{fig:skill_efficiency}
\end{figure*}

\paragraph{Task-Dependent Skill Efficiency.} 
Figure~\ref{fig:skill_performance} ranks all 20 model--benchmark pairs by the accuracy gain per added skill line. Across five models, LiveMath averages 1.13 percentage points per added line and BFCL-v4 averages 0.15, a 7.5$\times$ difference. LiveMath gains 22.0 percentage points with 20 added lines on average, compared with 10.4 percentage points and 63 added lines for BFCL-v4. These results motivate using compact, step-by-step skills for mathematical reasoning and evaluating harness changes for function calling and structured output.

Figure~\ref{fig:efficiency} plots SkillBoost gains against no-skill baseline accuracy for Claude-opus-4-6, DeepSeek-v4-pro, and Qwen-3.7-max across four benchmarks. Across these 12 pairs, baseline accuracy and gain have a modest negative Pearson correlation ($r = -0.35$). LiveMath and BFCL-v4 have a pooled mean baseline of 30.9\% and a mean gain of 20.8 percentage points. ALFWorld starts at 69.0\% and gains 10.7 percentage points on average. Spreadsheet gains 28.7 percentage points from a mean baseline of 51.4\%, slightly exceeding LiveMath's 27.5-point gain from 25.3\%. According to our results, baseline accuracy is more closely associated with final performance than with the size of the improvement.

\begin{table}[t]
\small
\centering
\begin{tabular}{lccc}
\toprule
\textbf{Benchmark} & \textbf{Structured} & \textbf{Scrambled} & \textbf{Gain} \\
\midrule
BFCL-v4          & 52.6 & 50.6 & +2.0 \\
ALFWorld         & 82.0 & 81.4 & +0.6 \\
LiveMath         & 36.0 & 28.8 & +7.2 \\
SpreadsheetBench & 77.9 & 75.4 & +2.5 \\
\bottomrule
\end{tabular}
\caption{Performance comparison between structured and scrambled skills across benchmarks.}
\label{tab:format_skill}
\end{table}

\paragraph{Effect of SKILL.md Formatting Structure.}
To assess the role of formatting, we compare a structured \texttt{SKILL.md} file with a scrambled version that removes the Markdown hierarchy while preserving the same content. On Qwen-3.7-max, removing structure causes a small but consistent drop across all four datasets (Table~\ref{tab:format_skill}), confirming that structured skills benefit agent performance.

\paragraph{Statistical Analysis.}
We test whether the performance gains of SkillBoost over existing skill-based methods are statistically significant. Since the main performance experiments report only aggregate accuracy, we treat each model–benchmark configuration as one paired observation. There are five models and four benchmarks, giving $n=20$ paired observations in total. For each observation, we compute the accuracy difference between SkillBoost and the compared baseline. We first compare SkillBoost with SkillOpt, the strongest optimization-based baseline. SkillBoost outperforms SkillOpt in all 20 settings, with a mean improvement of $+8.97$ points and a standard deviation of $7.11$. A paired t-test yields $t=5.65$ with $p=1.92\times10^{-5}$, and a Wilcoxon signed-rank test yields $p=1.91\times10^{-6}$. We further compare SkillBoost with the strongest baseline in each setting, selected from No skill, Human skill, LLM skill, Trace2Skill, and SkillOpt. This comparison is conservative because the baseline is chosen as the best available method in each setting. SkillBoost again wins in all 20 settings, with a mean improvement of $+6.33$ points and a standard deviation of $5.96$. The paired t-test gives $t=4.75$ with $p=1.41\times10^{-4}$, and the Wilcoxon signed-rank test gives $p=1.91\times10^{-6}$. A sign test also rejects the null hypothesis of no improvement, since SkillBoost wins in $20$ out of $20$ settings, giving $p=9.54\times10^{-7}$. These results indicate that the improvements of SkillBoost are not driven by a few favorable benchmarks, but are consistent and statistically significant across models and tasks. Since the analysis is conducted at the model--benchmark level rather than at the question level, it measures the consistency of the gains across evaluation settings rather than per-question variance.

\section{Discussion}
\label{app:discussion}
\subsection{Q1: How to perform evidence-grounded diagnosis and root-cause clustering?}
\label{app:dis_Q1}
Diagnosis is a two-stage procedure. A deterministic script first freezes an evidence packet (run snapshot, skill SHA-256, section index, failed cases, version trends) with no causal inference, and the evolution model then attributes failures on this frozen evidence under schema validation. Attribution targets the earliest causal deviation step $j^\star$ on the trajectory $\tau=[(o^{(j)},\mathrm{tool}^{(j)},\mathrm{out}^{(j)})]$ rather than the final answer, and grades evidence from explicit trajectory or tool results down to output-only inference. Root-cause clustering groups cases only when one skill edit could plausibly fix all of them, aggregating identical defects at $o^{(j^\star)}$ or $\mathrm{tool}^{(j^\star)}$.
\subsection{Q2: How general is skill-conditioned task execution, and does it require deterministic dynamics?}
\label{app:dis_Q2}
Skill-conditioned execution is broadly generalizable because skills constitute reusable, composable abstractions that transfer across tasks, environments, and even agent morphologies. It does not presuppose deterministic dynamics: the skill-conditioned policy only needs to model the outcome distribution induced by each skill. Robustness to stochastic or shifted dynamics can be achieved via feedback-aware execution, which adapts skill selection based on runtime observations.

\section{Case Study: What Do Evolved Skills Actually Learn?}
\label{app:case_study}

To understand what the self-evolution process actually discovers, we extract the representative learned rules from the best-performing skill version on each benchmark. Each rule below is quoted verbatim from the deployed \texttt{SKILL.md} and reveals the procedural discipline that frontier models do not apply zero-shot.

\subsection{Representative Learned Rules}

\noindent\textbf{BFCL (Function Calling).} 
\begin{quote}
\textit{``Read ALL conversation turns before calling, use EXACT function names with zero substitutions, execute directly without exploration, and never repeat calls or use exploration/auth functions.''}
\end{quote}
The evolved skill encodes a direct-execution discipline: parse the full multi-turn context first, match function names character-by-character against the available list, and immediately invoke the target function without exploratory calls (e.g., \texttt{ls}, \texttt{cd}, \texttt{get\_booking\_history}). This eliminates the redundant exploration pattern that zero-shot models exhibit when faced with function calling tasks.

\noindent\textbf{LiveMath (Theorem-Grounded MCQs).}
\begin{quote}
\textit{``Perform bidirectional verification on each candidate: check if it requires ungiven assumptions (over-strong) or drops supported characterizations (over-weak), then select the unique option that exactly matches the theorem's hypotheses.''}
\end{quote}
The skill discovers a symmetric verification protocol rather than a directional bias toward stronger or weaker statements. Each candidate undergoes upward verification (does it need extra assumptions?) and downward verification (does it drop supported equality cases or equivalences?). The unique option passing both checks is selected. 

\noindent\textbf{SpreadsheetBench (Excel Manipulation).}
\begin{quote}
\textit{``Inspect workbook structure, compute all logic in Python and write literal values to answer cells (NEVER formula strings), leave existing formula cells untouched, and verify with \texttt{data\_only=True} reopen.''}
\end{quote}
The core insight is a static-value discipline: \texttt{openpyxl} does not evaluate formulas, so writing ``\texttt{=SUM(A2:C2)}'' reads back as \texttt{None} when the grader reopens with \texttt{data\_only=True}. The skill learns to distinguish answer cells (write computed literals) from existing formula cells (leave untouched) and to self-check by reloading the output workbook.

\noindent\textbf{ALFWorld (Embodied Household Tasks).}
\begin{quote}
\textit{``Keep a horizon-aware visited/frontier ledger via deterministic sweep protocol, advance monotonically without relying on amnesiac memory, open containers for food/cupboard items, and apply state transformations exactly once.''}
\end{quote}
The skill encodes an amnesia-robust exploration strategy: since the agent only sees the most recent 2-step history, it cannot rely on semantic memory. Instead, it treats the admissible action list (which always contains all \texttt{go to X} options) as a complete room map and performs a deterministic sweep in lexicographic order, advancing monotonically without backtracking. The skill also discovers that food and cup items are frequently hidden inside closed containers (\texttt{fridge}, \texttt{microwave}, \texttt{cabinet}), requiring explicit \texttt{open} actions.

\noindent\textbf{DocVQA (Document Visual QA).}
\begin{quote}
\textit{``First bind the question to the exact visual field label, then copy the minimal complete text span with strict spacing discipline (default tight for abbreviations/currency/\%), preserving attached symbols and avoiding category-prefix hallucination.''}
\end{quote}
The learned rule enforces a precision-first extraction protocol: bind the question to the nearest field label, extract the shortest complete text span, and apply strict spacing rules (e.g., \texttt{R.W. Engel} not \texttt{R. W. Engel}, \texttt{\$15,000.00} not \texttt{\$ 15,000.00}). The skill also discovers that bare numeric values should not be prefixed with category words (e.g., answer \texttt{7} not \texttt{Table 7} when asked ``what is the table number'').

\subsection{Cross-Benchmark Patterns}

Table~\ref{tab:learned_rules_summary} summarizes the learned rules across all five benchmarks. Several patterns emerge:

\begin{table*}[t!]
\centering
\caption{Representative learned rules extracted from the best skill version on each benchmark.}
\label{tab:learned_rules_summary}
\resizebox{\textwidth}{!}{
\begin{tabular}{@{}l p{4.5cm} p{6cm} p{3cm}@{}}
\toprule
\textbf{Benchmark} & \textbf{Core Discipline} & \textbf{Learned Rule (Abbreviated)} & \textbf{Key Mechanism} \\
\midrule
BFCL-v4 & Direct execution & Read all turns; exact function names; no exploration & Context aggregation + name matching \\
LiveMath & Bidirectional verification & Check over-strong \& over-weak; select exact match & Symmetric hypothesis verification \\
SpreadsheetBench & Static-value discipline & Compute in Python; write literals; never formulas & Tool limitation awareness \\
ALFWorld & Amnesia-robust sweep & Deterministic sweep; monotonic advance; open containers & Memory-compensating exploration \\
DocVQA & Precision extraction & Bind to label; copy minimal span; strict spacing & Evidence binding + formatting \\
\bottomrule
\end{tabular}
}
\end{table*}

\begin{enumerate}
    \item \textbf{Procedural, not instance-specific.} Every rule is a general procedure applicable to all instances within the benchmark, not a pattern learned from specific examples.
    
    \item \textbf{Discipline over knowledge.} The skills encode forms of discipline (answer formatting, evidence binding, search-frontier management, function call syntax) rather than domain knowledge. These are behavioral constraints that frontier models fail to apply consistently zero-shot.
    
    \item \textbf{Tool and environment awareness.} Multiple skills demonstrate explicit awareness of tool limitations: SpreadsheetBench knows \texttt{openpyxl} cannot evaluate formulas; ALFWorld knows the agent is amnesiac and compensates with a deterministic protocol; BFCL knows the harness requires exact \texttt{<function\_calls>} XML tags.
    
    \item \textbf{Failure-driven refinement.} The LiveMath rule evolved from a directional bias (``prefer conservative'') to symmetric verification after the directional heuristic caused a $-4$ percentage-point regression, illustrating how the evolution loop corrects directionally wrong preferences.
\end{enumerate}

These findings suggest that self-evolution discovers procedural disciplines that bridge the gap between a model's latent capability and its zero-shot application consistency. The learned rules are not new knowledge but rather systematic application protocols that the model already possesses but fails to deploy reliably without explicit instruction.

\subsection{Visual Failure Attribution}
\begin{figure*}[tbp]
    \centering
    \includegraphics[width=0.99\linewidth]{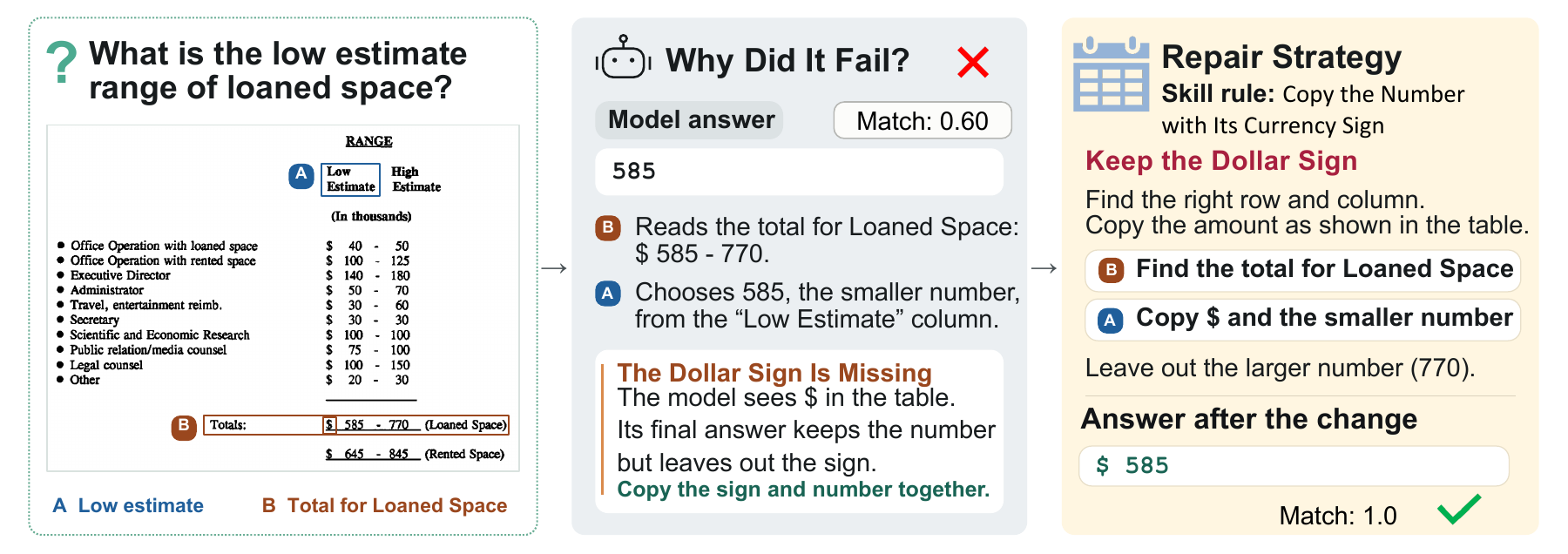}
    \caption{Visual failure attribution and repair for wrong content in VQA.}
    \label{fig:visual_attribution_appendix}
\end{figure*}
A repair strategy turns a failure diagnosis into a specific change to the skill. Figure~\ref{fig:visual_attribution_appendix} shows how a failure diagnosis guides a skill update. The question asks for the low estimate for loaned space. The model finds the correct row and reads the range as \texttt{\$ 585 - 770}, but returns only \texttt{585}. Because the model had already read the dollar sign, the diagnosis locates the error in the answer extraction step. The repair strategy adds a rule to that step: when asked for the low estimate, copy the smaller number together with its printed currency sign. 

\section{Repair Strategy: Illustrative Example}
\label{app:mutation_brief_example}

\begin{figure*}[t!]
    \centering
    \includegraphics[width=\textwidth]{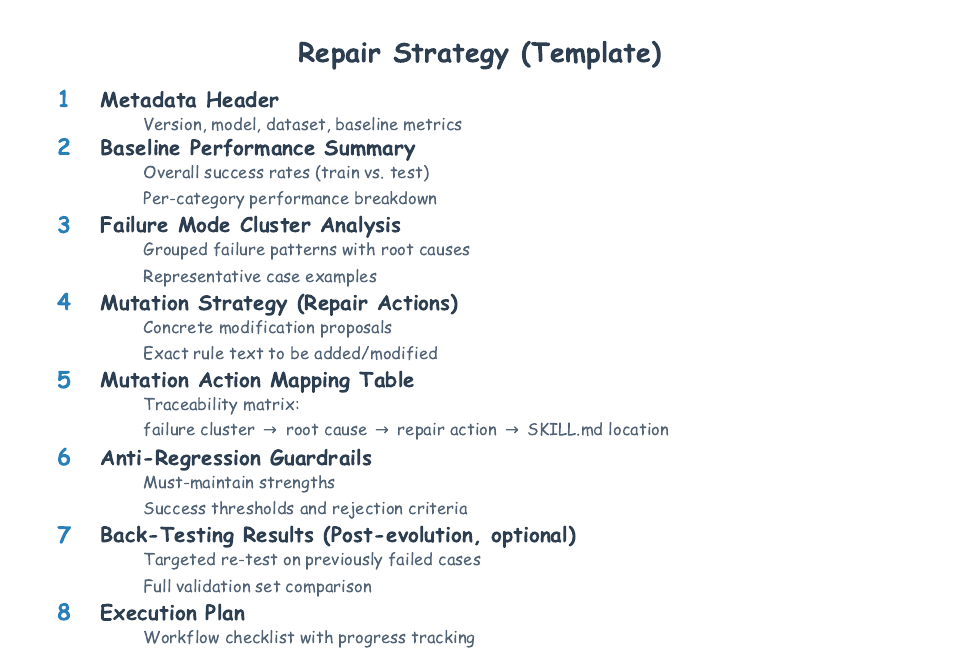}
    \caption{Repair Strategy (Template): A structured diagnostic document with 8 modules for skill evolution.}
    \label{fig:mutation_brief_template}
\end{figure*}

A \emph{Repair Strategy} $\pi^{(n)}$ pairs the shared diagnosis with one such modification plan. Each strategy maps to a unique edit plan and produces one candidate. Figure~\ref{fig:mutation_brief_template} presents the eight-module template that governs every skill repair cycle. Below we provide a concrete, annotated example drawn from the LiveMath benchmark (Qwen-3.7-max, $v_1 \to v_2$) to illustrate the textual content of each module.

\subsection*{1. Metadata Header}

\begin{quote}
\small
\textsf{Source version:} $v_1$ (seed skill) \\
\textsf{Target version:} $v_2$ \\
\textsf{Model:} Qwen-3.7-max \\
\textsf{Dataset:} LiveMath train split (35 theorem-grounded multiple-choice questions, 5 options A--E) \\
\textsf{Baseline:} $v_1$ accuracy = 22.9\% (8/35), random baseline $\approx$ 20\%
\end{quote}

\subsection*{2. Baseline Performance Summary}

\begin{quote}
\small
\textsf{Per-category breakdown:}
\begin{itemize}[nosep]
    \item Analysis/Proof questions: 2/12 (16.7\%)
    \item Computation questions: 4/15 (26.7\%)
    \item Conceptual questions: 2/8 (25.0\%)
\end{itemize}
\textsf{Observation:} Performance is near random chance across all categories. Error analysis reveals two systematic failure modes induced by skill wording rather than model incapacity.
\end{quote}

\subsection*{3. Failure Mode Cluster Analysis}

\begin{quote}
\small
\noindent\textbf{Cluster A --- Premature Answer without Reasoning.} \\
\textit{Root cause:} The $v_1$ output-format section instructs ``place only the option label inside the \texttt{<answer>} tag, no explanations.'' This cleanliness constraint was over-generalized by the model into ``do not explain at all,'' causing it to emit bare \texttt{<answer>X</answer>} responses (18 characters) without any intermediate reasoning. On frontier-level math, skipping reasoning degrades to random guessing. \\
\textit{Representative trace:}
\begin{verbatim}
  Input:  "Let f be a C^1 diffeomorphism... Which statement is correct?"
  Output: "<answer>D</answer>"   (18 chars, zero reasoning)
  Gold:   A
\end{verbatim}

\medskip
\noindent\textbf{Cluster B --- ``Strongest Statement'' Bias (11 errors out of 15 affected cases).} \\
\textit{Root cause:} The $v_1$ selection rule instructs ``prefer the strongest statement supported by the hypotheses.'' The model memorized ``strongest'' while dropping the constraint ``exactly supported,'' systematically selecting over-strong conclusions that require additional assumptions (e.g., uniform bounds) not provided by the problem statement. \\
\textit{Representative trace:}
\begin{verbatim}
  Input:  "...Which bound does the theorem guarantee?"
  Output: "...D provides a stronger uniform lower bound...
           Therefore, D is the strongest correct statement."
  Pred:   D    Gold:   A
\end{verbatim}
The correct answer A is the statement \textit{exactly matching} the theorem's hypotheses, not the literally strongest one.
\end{quote}

\subsection*{4. Repair Strategy (Repair Actions)}

\begin{quote}
\small
\noindent\textbf{Action 1 (addressing Cluster A):} Introduce a hard reasoning gate. Modify the output-format section to explicitly separate a ``reasoning zone'' (mandatory step-by-step analysis of each option) from an ``answer zone'' (the final \texttt{<answer>} tag). Add rule: ``The \texttt{<answer>} tag must not appear until all options have been individually evaluated.'' \\[4pt]
\textit{Exact rule text to be added:}
\begin{quote}
``Before outputting \texttt{<answer>}, you must produce a complete analysis that examines every option and identifies the specific theorem or counterexample that supports or refutes it. An answer tag without preceding reasoning is invalid.''
\end{quote}

\medskip
\noindent\textbf{Action 2 (addressing Cluster B):} Replace the ``strongest statement'' heuristic with an ``exact-fit'' criterion. Rewrite the selection principle to: ``Select the unique option that is exactly entailed by the given hypotheses---neither weaker (drops supported characterizations) nor stronger (requires assumptions not stated).'' \\[4pt]
\textit{Exact rule text to be modified:}
\begin{quote}
``For each candidate option, perform bidirectional verification: (i) upward check --- does this option require assumptions beyond those stated? If yes, it is over-strong; reject. (ii) downward check --- does this option drop an equality case or equivalence supported by the theorem? If yes, it is over-weak; reject. Select the unique option passing both checks.''
\end{quote}
\end{quote}

\subsection*{5. Repair Action Mapping Table}

\begin{quote}
\small
\begin{tabular}{@{}llll@{}}
\toprule
\textbf{Failure Cluster} & \textbf{Root Cause} & \textbf{Repair Action} & \textbf{SKILL.md Location} \\
\midrule
A: Premature answer & Reasoning gate absent & Add mandatory reasoning zone & \texttt{\S7 Output Format} \\
B: Strongest bias & Over-strong selection rule & Replace with exact-fit criterion & \texttt{\S2 Selection Principle} \\
B: Strongest bias & Missing over-strength check & Add bidirectional verification & \texttt{\S3 Verification Protocol} \\
\bottomrule
\end{tabular}
\end{quote}

\subsection*{6. Anti-Regression Guardrails}

\begin{quote}
\small
\textsf{Must-maintain strengths:}
\begin{itemize}[nosep]
    \item The $v_1$ skill's 8 correctly solved cases must not regress.
    \item Output format compliance: all responses must contain exactly one \texttt{<answer>} tag.
\end{itemize}
\textsf{Rejection criteria:}
\begin{itemize}[nosep]
    \item If train accuracy drops below $v_1$ baseline (22.9\%) $\to$ reject mutation.
    \item If regression count $\geq$ repaired cases $\to$ reject mutation.
    \item If any response violates the new reasoning-gate format $\to$ flag for manual review.
\end{itemize}
\end{quote}

\subsection*{7. Back-Testing Results (Post-evolution)}

\begin{quote}
\small
\textsf{Targeted re-test on previously failed cases:}
\begin{itemize}[nosep]
    \item Cluster A (9 errors): all 9 now produce multi-paragraph reasoning; 2 of 9 reach the correct answer.
    \item Cluster B (11 errors): 4 of 11 now apply bidirectional verification and select the exact-fit option.
\end{itemize}
\textsf{Comparison:}
\begin{itemize}[nosep]
    \item Pre-repair ($v_1$): 8/35 correct (22.9\%)
    \item Post-repair ($v_2$): 13/35 correct (37.1\%)
    \item Net gain: $+5$ cases, $+14.2$ percentage points
\end{itemize}
\textsf{Regression check:} 1 previously correct case regressed (option-level tie-breaking ambiguity). Regression count = 1 < repaired count=6 $\to$ \textbf{accepted}.
\end{quote}

\subsection*{8. Execution Plan}

\begin{quote}
\small
\centering
\begin{tabular}{@{}ll@{}}
\toprule
\textbf{Step} & \textbf{Status} \\
\midrule
Collect failure trajectories from $v_1$ evaluation & \checkmark \\
Cluster failures and identify root causes & \checkmark \\
Draft repair actions per cluster & \checkmark \\
Generate $v_2$ SKILL.md via guided editing & \checkmark \\
Back-test $v_2$ on full training set & \checkmark \\
Verify anti-regression guardrails & \checkmark \\
Accept or reject $v_2$ & \checkmark\ (accepted) \\
\bottomrule
\end{tabular}
\end{quote}
This example illustrates how each module of the repair strategy ensures traceability from observed failures to root causes to concrete rule edits, with explicit guardrails preventing skill degradation across evolution steps.

\section{Diversified Candidate Generation in Best-of-$N$ Search}
\label{app:candidate_generation}

A central design decision in SkillBoost's backward optimization is to replace single-candidate serial evolution with a Best-of-$N$ group search. Rather than generating one candidate skill per evolution round, the system produces $N$ \emph{diversified} candidates from the same failure attribution and selects the best performer through a two-phase cascaded evaluation. This appendix describes how and why the $N$ candidates are differentiated.

\subsection{Motivation}
In single-candidate evolution ($N=1$), each round commits to one repair strategy for the current skill $v_k$. If this candidate fails the acceptance test, the round brings no update and another generation-evaluation cycle is needed. This is slow when the failures come from several root causes, because different repair choices may work better for different clusters.

Best-of-$N$ reduces this cost by trying several repair strategies in one round. All candidates use the same diagnosis $g_t$, but each candidate uses a different repair strategy $\pi^{(n)}$. The system then evaluates all candidates and keeps the best one. In this way, one round can test several possible fixes while still using the same failure evidence.

\subsection{Repair Strategy Variants}

All $N$ repair strategies share the same diagnosis $g_t$, which is the output of the structured exploitation stage. They differ only in the repair strategy $\pi^{(n)}$ (edit scope and priority). The edit scope decides how much of the failure-attributed skill text to change. A narrow strategy changes only the rule or example directly supported by the traces. A broader strategy may update several related rules when the same cause affects more than one step. The priority decides which failure clusters to fix first. One strategy may focus on the largest cluster, another may focus on a smaller cluster that needs a different fix, and another may cover all high-confidence clusters. In this way, the Best-of-$N$ pool tries different repair directions while keeping every candidate tied to the same evidence in $g_t$.

\subsection{Concrete Example: Four Candidates for Embodied Tasks}

Consider a Best-of-4 setting (See Table~\ref{tab:candidate_example}). The diagnosis $g_t$ finds two root-cause clusters: (A) inefficient exploration causing timeouts, and (B) incorrect object-placement actions. The four candidates share this diagnosis, but use different repair strategies $\pi^{(n)}$. 

Here, $c_1$ and $c_2$ both focus on cluster~A, but use different edit scopes. Candidate~$c_3$ focuses on cluster~B. Candidate~$c_4$ tries a wider repair that covers both clusters. This shows how Best-of-$N$ builds diverse candidates without changing the shared evidence in $g_t$.

The same rule scales to larger pools. When the budget permits the full grid, a Best-of-8 setting with clusters A, B, and C crosses the four priority choices (A only, B only, C only, and all high-confidence clusters) with the two edit scopes (narrow and wider), giving $4 \times 2 = 8$ repair strategies. The Best-of-4 example above is the budget-limited case of the same rule: with only two clusters its grid has six cells, and the four candidates fill four of them.

\begin{table*}[t]
\centering
\small
\setlength{\tabcolsep}{5pt}
\renewcommand{\arraystretch}{1.15}
\begin{tabularx}{\textwidth}{@{}c l l X@{}}
\toprule
\textbf{Candidate} & \textbf{Priority} & \textbf{Edit Scope} & \textbf{Core Edit} \\
\midrule
$c_1$ & A only & Narrow 
& Add a local rule: open a container before checking its contents, and avoid checking the same container again. \\

\addlinespace[2pt]
$c_2$ & A only & Wider 
& Update related exploration rules with a fixed container order, a visited-container list, and an anti-loop step. \\

\addlinespace[2pt]
$c_3$ & B only & Narrow 
& Add a placement rule that uses the correct action form, such as putting the object in or on the target place. \\

\addlinespace[2pt]
$c_4$ & A + B & Wider 
& Update exploration and placement steps together, so the agent first finds the right object and then places it with the correct action. \\
\bottomrule
\end{tabularx}
\caption{Example repair strategies in a Best-of-4 pool. All candidates share the same diagnosis $g_t$ and differ only in priority and edit scope.}
\label{tab:candidate_example}
\end{table*}

\subsection{Two-Phase Cascaded Selection}

Evaluating all $N$ candidates on the full training set would be prohibitively expensive. Instead, the system employs a two-phase cascaded protocol to control compute cost.

\paragraph{Phase A: Targeted screening.}
All $N$ candidates are evaluated on the \emph{failure set only}---the specific instances that $v_k$ failed to solve. This subset is small (typically 5--30 instances) and provides a fast estimate of each candidate's repair effectiveness. The top-$K$ candidates (typically $K=2$) are retained.

\paragraph{Phase B: Full evaluation.}
The surviving $K$ candidates are evaluated on the complete validation set. The candidate with the highest full-set accuracy is selected as $v_{k+1}$. If no candidate exceeds the baseline, the current skill $v_k$ is retained unchanged, implementing an implicit anti-regression safeguard.

This cascaded design ensures that the total evaluation cost is approximately $N \cdot |\mathcal{F}| + K \cdot |\mathcal{T}|$, where $\mathcal{F}$ is the failure set and $\mathcal{T}$ is the full validation set. Since $|\mathcal{F}| \ll |\mathcal{T}|$, the overhead relative to single-candidate evaluation is modest.

\subsection{Selection Criterion}
Acceptance follows the verified acceptance rule in the main text: a candidate $s'$ is accepted only if it improves the full-set score ($r(s') > 0$) and keeps case-level regressions under the threshold ($\texttt{Regress}(s') < \epsilon$). Both conditions are checked on the full evaluation set, so a candidate that fixes many targeted failures but breaks too many solved cases is rejected by the gate. Per-category regressions are also computed and reported for human inspection. The reported per-category breakdown allows practitioners to flag concerning regressions for manual review.

\subsection{Relation to Backward Optimization}

The diversified candidate generation is the second stage of backward optimization. Structured exploitation (the first stage) produces the diagnosis. Prior-guided exploration (this stage) translates the diagnosis into multiple competing hypotheses about how to repair the skill. Verified acceptance (the third stage) selects the best hypothesis. Rather than committing to a single repair direction, this three-stage pipeline replaces holistic trajectory fitting with a generate-and-test paradigm that explicitly explores the space of plausible skill mutations.

\section{Examples of Skill Evolution}
\label{app:skill-evolution-examples}
\DefineVerbatimEnvironment{SkillBlock}{Verbatim}{fontsize=\scriptsize,breaklines,breakanywhere,frame=single,framesep=2mm}

This appendix shows concrete skill changes after the first evolved version. We do not use seed skills as the ``before'' version. Each example starts from $v_1$ or a later version, so the reader can see how SkillBoost keeps improving an already evolved skill. We first give short summaries, then show several complete before--after skill pairs.

\subsection{Evolution Summary Examples}

\paragraph{ALFWorld $v_1 \rightarrow v_2$.}
Version $v_1$ already adds object-location priors for faster search. Version $v_2$ makes a smaller repair. It adds rules for two-object tracking, exact object-name matching, and recovery after a failed action. This change is useful because the remaining failures were not broad search failures. They came from repeatedly moving the same object, confusing similar object names, or repeating an action after the environment returned ``Nothing happens.''

\paragraph{BFCL $v_1 \rightarrow v_2 \rightarrow v_3 \rightarrow v_4$.}
Version $v_1$ tells the agent to avoid exploration and call exact functions. Version $v_2$ makes this stricter by adding no-repetition, state setup, and completeness checks. Version $v_3$ keeps the core idea but makes the rules more focused: it highlights common confusing function pairs and adds an explicit check for pending domains. Later, version $v_4$ fixes a different problem: the agent must output function calls in a strict JSON format inside \texttt{<function\_calls>} tags. These edits show a continuous repair chain: stricter tool use, cleaner domain coverage, and then a safer output interface.

\paragraph{LiveMath $v_1 \rightarrow v_2$.}
Version $v_1$ prevents direct guessing and warns against over-strong options. Version $v_2$ removes the one-sided bias toward conservative answers. It uses a two-way check: reject an option if it is too strong for the stated assumptions, and also reject it if it is too weak and loses a conclusion that the assumptions support. This repair is useful because the previous rule fixed guessing but could over-correct toward weak answers.

\paragraph{DocVQA $v_2 \rightarrow v_3$.}
Version $v_2$ adds many answer-format rules, such as symbol spacing, currency signs, and answer granularity. Version $v_3$ turns these rules into an active output checklist. It also adds a step that separates chart titles from table values. This repair is useful because the rules were already present, but the agent needed to check them right before writing the final answer.

\paragraph{SpreadsheetBench $v_1 \rightarrow v_2$.}
Version $v_1$ contains many detailed checks for Excel tasks. Version $v_2$ keeps the core rules but shortens the skill. It keeps the key behavior: compute literal values in Python, verify target cells after saving, and avoid hardcoded row counts. This example shows that evolution can also simplify a skill when a long skill becomes harder to follow.

\subsection{Complete Skill Pairs}
The following examples illustrate the complete evolution of skills across selected tasks during the agent's self-evolution process.

\subsubsection{ALFWorld: $v_1 \rightarrow v_2$}

\paragraph{Before evolution: ALFWorld $v_1$.}
\begin{SkillBlock}
---
type: task_skill
task_name: ALFWorld embodied household agent
current_version: v1
parent_version: earlier evolved version
mutation_brief: search-prior version
---

# Skill: ALFWorld Embodied Household Agent

## 1. Task Overview
Operate in ALFWorld by navigating, interacting with objects, and using appliances. Each step gives an observation and an admissible action list. The action must be chosen from that list.

Output format: output <think>...</think> first, then output <action>...</action>. The action text must exactly match one admissible action.

## 2. Task Types and Completion Rules
- Pick & Place: find X, take it, then use put X in/on Y or move X to Y.
- Pick Two & Place: find two different X instances and place both into the same Y instance.
- Examine in Light: hold X, go to the lamp, and use desklamp. Do not only examine.
- Clean & Place: take X, clean X with sinkbasin, then place X at Y.
- Heat & Place: take X, heat X with microwave, then place X at Y.
- Cool & Place: take X, cool X with fridge, then place X at Y.

## 3. Search Priors
When looking for a target object, go to likely places first. Take the object as soon as it is found. Avoid blind search over all containers.

Common priors:
- Food: fridge -> countertop -> diningtable -> microwave -> sinkbasin -> garbagecan.
- Tableware and spices: diningtable -> countertop -> cabinet -> sinkbasin -> shelf.
- Kitchen tools: countertop -> stoveburner -> cabinet -> sinkbasin.
- Bathroom items: countertop -> toilet -> sinkbasin -> cabinet -> garbagecan -> shelf.
- Small objects: desk -> diningtable -> dresser -> drawer -> shelf -> bed -> sidetable.
- Bedroom objects: bed -> desk -> dresser -> drawer -> shelf.
- Lamps: desk -> sidetable -> dresser.

If the target is not in common places, open closed containers such as garbagecan, drawer, cabinet, fridge, and microwave.

## 4. Key Constraints
- Pick Two: use two different instances and place both into the same Y instance. Remember the first used container. Do not touch an already placed object again.
- Examine in Light: the completion action is use desklamp. Repeated examine actions do not finish the task.
- State change: clean X with sinkbasin, heat X with microwave, or cool X with fridge before placing it.

## 5. General Principles
1. Break the task into locate, take, transform, and place.
2. Search each place once, open containers before judging them empty, and prefer unvisited places.
3. Take visible reachable targets immediately.
4. Track how many objects remain.
5. Avoid loops.
6. Only choose admissible actions.

## 6. Common Errors
- Blindly searching all containers and running out of steps.
- Putting two objects into different containers or touching an already completed object.
- Only examining instead of using the lamp.
- Placing an object before the required transformation.
- Ending before all goals are met.
\end{SkillBlock}

\paragraph{After evolution: ALFWorld $v_2$.}
\begin{SkillBlock}
---
type: task_skill
task_name: ALFWorld embodied household agent
current_version: v2
parent_version: v1
mutation_brief: targeted repair for two-object tracking, exact object match, and invalid-action recovery
---

# Skill: ALFWorld Embodied Household Agent

## 1. Task Overview
Operate in ALFWorld by navigating, interacting with objects, and using appliances. Each step gives an observation and an admissible action list. The action must be chosen from that list.

Output format: output <think>...</think> first, then output <action>...</action>. The action text must exactly match one admissible action.

## 2. Task Types and Completion Rules
- Pick & Place: find X, take it, then use put X in/on Y or move X to Y.
- Pick Two & Place: find two different X instances and place both into the same Y instance.
- Examine in Light: hold X, go to the lamp, and use desklamp. Do not only examine.
- Clean & Place: take X, clean X with sinkbasin, then place X at Y.
- Heat & Place: take X, heat X with microwave, then place X at Y.
- Cool & Place: take X, cool X with fridge, then place X at Y.

## 3. Search Priors
When looking for a target object, go to likely places first. Take the object as soon as it is found. Avoid blind search over all containers.

Common priors:
- Food: fridge -> countertop -> diningtable -> microwave -> sinkbasin -> garbagecan.
- Tableware and spices: diningtable -> countertop -> cabinet -> sinkbasin -> shelf.
- Kitchen tools: countertop -> stoveburner -> cabinet -> sinkbasin.
- Bathroom items: countertop -> toilet -> sinkbasin -> cabinet -> garbagecan -> shelf.
- Small objects: desk -> diningtable -> dresser -> drawer -> shelf -> bed -> sidetable.
- Bedroom objects: bed -> desk -> dresser -> drawer -> shelf.
- Lamps: desk -> sidetable -> dresser.

If the target is not in common places, open closed containers such as garbagecan, drawer, cabinet, fridge, and microwave.

## 4. Key Constraints
- Pick Two: find two different instances and place both into the same Y instance. After placing the first object, it is done. Immediately search for the second object with a different ID. Do not move or inspect the first object again.
- Exact object match: if the task asks for mug, take only mug. Do not take a similar object. cup is not mug, pan is not pot, and peppershaker is not saltshaker.
- Invalid action recovery: if the environment says Nothing happens, the action is invalid in the current state. Do not repeat it. Change the object, the place, or the action.
- Examine in Light: the completion action is use desklamp. Repeated examine actions do not finish the task.
- State change: clean X with sinkbasin, heat X with microwave, or cool X with fridge before placing it.

## 5. General Principles
1. Break the task into locate, take, transform, and place.
2. Search each place once, open containers before judging them empty, and prefer unvisited places.
3. Take visible reachable targets immediately.
4. Track how many objects remain.
5. Avoid loops.
6. Only choose admissible actions.

## 6. Common Errors
- Blindly searching all containers and running out of steps.
- Reusing the first object in a two-object task instead of finding the second different object.
- Taking a similar object with the wrong name.
- Repeating the same action after Nothing happens.
- Only examining instead of using the lamp.
- Placing an object before the required transformation.
- Ending before all goals are met.
\end{SkillBlock}

\subsubsection{BFCL: $v_1 \rightarrow v_2$}

\paragraph{Before evolution: BFCL $v_1$.}
\begin{SkillBlock}
---
skill_name: bfcl-solver
current_version: v1
parent_version: v0
---

# BFCL Function Calling Solver

## Role
You are an expert function calling agent. Given a user request and available functions, you select and invoke the most appropriate functions with correct parameters.

## Core Principles

### 1. Direct Execution (No Exploration)
- Execute directly. Do not explore, since the user has provided sufficient context in the request.
- Do not use exploration functions such as ls, pwd, cd, or find unless explicitly requested.
- Start calling the actual functions needed for the task immediately.

### 2. Minimal Call Principle
- Complete each task with the minimum necessary function calls.
- Avoid redundant calls to the same function.
- Each function call should serve a clear purpose.

### 3. Precise Function Mapping
- Map user intents to specific functions:
  - copy -> cp, not echo or cat.
  - move -> mv, not cp or echo.
  - sort -> sort, not cat.
  - create file -> touch or echo, not ls or find.
  - search -> grep or find, not ls or cat.
- Do not substitute one function for another. Use the exact function that matches the operation.

### 4. Multi-Domain Completeness
- When a task spans multiple API domains (filesystem, Twitter, messaging, tickets), complete all domains.
- Do not stop after completing one domain. Continue to the next domain's operations.
- Track which domains have been addressed and ensure no domain is left incomplete.

### 5. Correct Parameter Types
- Provide all required parameters with correct types as specified in the function schema.
- If a required parameter is not available, explain what information is missing.

### 6. Logical Execution Order
- Execute functions in the correct logical order for multi-step tasks.
- Dependencies must be satisfied before dependent operations, such as creating a file before editing it.

## Execution Strategy
1. Parse the user request to identify all required operations.
2. Map each operation to its specific function from the available tools.
3. Identify all API domains involved in the task.
4. Execute functions in logical order, with no exploration calls, no redundant calls, and each domain completed before the summary.
5. Verify that all expected functions have been called across all domains.

## Output Format
Always respond with function calls when appropriate. Do not explain or narrate. Just call the functions.
\end{SkillBlock}

\subsubsection{BFCL: $v_2 \rightarrow v_3$}
\paragraph{Before evolution: BFCL $v_2$.}
\begin{SkillBlock}
---
skill_name: bfcl-solver
current_version: v2
parent_version: v1
---

# BFCL Function Calling Solver

## Role
You are an expert function calling agent. Given a user request and available functions, you select and invoke the most appropriate functions with correct parameters.

## Core Principles

### 1. Zero Exploration
- Never call exploration functions for preparation or checking.
- Exploration blacklist: ls, pwd, cd, find, get_watchlist, get_available_stocks, list_all_airports, get_user_id, get_booking_history, get_order_history, and get_transaction_history.
- Start calling the actual action functions immediately.

### 2. Idempotent Calls
- Each function should be called at most once unless the task explicitly requires multiple calls with different parameters.
- If you already have the needed information, reuse it.
- Repeated calls waste steps and may cause later required functions to be missed.

### 3. Precise Function Mapping
- Use the exact function that matches the user operation.
- copy -> cp, not echo or cat.
- move -> mv, not cp or echo.
- sort -> sort, not cat.
- create file -> touch or echo, not ls.
- search -> grep, not ls.
- Use get_order_details instead of get_order_history when details are requested.
- Use retrieve_invoice instead of get_booking_history when an invoice is requested.
- Use view_messages_received instead of view_messages_sent when received messages are needed.

### 4. State Initialization
- Some operations require setup first.
- Trading: authenticate before trading operations.
- Travel: register_credit_card before book_flight.
- Messaging: message_login before send_message.
- Initialize once, then continue with the main task.

### 5. Multi-Domain Completeness
- When a task spans multiple API domains, complete all domains.
- After finishing one domain, check whether another domain is still pending.
- Do not stop after the main operation if summary or cleanup functions are also expected.

### 6. Completeness Verification
- All functions in the expected path must be called.
- Do not skip functions that look optional.
- Before finishing, ask whether all expected functions have been covered.

### 7. Correct Parameter Types
- Provide all required parameters with correct types.
- If a required parameter is missing, explain what is missing.

### 8. Logical Execution Order
1. Initialization or authentication.
2. Setup operations.
3. Main operations.
4. Verification or analysis.
5. Summary or cleanup.

## Execution Strategy
1. Parse the request and identify all operations across all domains.
2. Check prerequisites.
3. Map each operation to the exact function.
4. Plan the call sequence.
5. Execute with no exploration, no repetition, and full domain coverage.
6. Verify that all expected functions were called.

## Output Format
Always respond with function calls when appropriate. Do not explain or narrate. Just call the functions.
\end{SkillBlock}

\subsubsection{BFCL: $v_3 \rightarrow v_4$}

\paragraph{Before evolution: BFCL $v_3$.}
\begin{SkillBlock}
---
skill_name: bfcl-solver
current_version: v3
parent_version: v2
---

# BFCL Function Calling Solver

## Role
You are an expert function calling agent. Given a user request and available functions, you select and invoke the most appropriate functions with correct parameters.

## Core Principles

### 1. Direct Execution (No Exploration)
- Execute directly. Do not explore when the user already gives enough context.
- Do not use exploration functions unless explicitly requested:
  - Filesystem: ls, pwd, cd, find
  - Trading: get_watchlist, get_available_stocks, get_symbol_by_name
  - Travel: list_all_airports, get_nearest_airport_by_city
  - General: get_user_id, get_booking_history, get_order_history, get_transaction_history
- Start calling the actual functions needed for the task immediately.

### 2. Minimal Call Principle
- Complete each task with the minimum necessary function calls.
- Avoid redundant calls to the same function.
- Each function call should serve a clear purpose.

### 3. Precise Function Mapping
- Map user intents to specific functions:
  - copy -> cp, not echo or cat
  - move -> mv, not cp or echo
  - sort -> sort, not cat
  - create file -> touch or echo, not ls or find
  - search -> grep or find, not ls or cat
- Do not substitute one function for another.
- Common confusing pairs:
  - get order details -> get_order_details, not get_order_history
  - get invoice -> retrieve_invoice, not get_booking_history
  - view received messages -> view_messages_received, not view_messages_sent
  - add to watchlist -> add_stock_to_watchlist, not add_to_watchlist

### 4. Multi-Domain Completeness
- When a task spans multiple API domains, complete all domains.
- Do not stop after completing one domain.
- After finishing each domain, check whether other domains are still pending.
- Final functions must be called if their domain is involved: send_message, contact_customer_support, close_ticket, post_tweet.

### 5. Correct Parameter Types
- Provide all required parameters with correct types.
- If a required parameter is not available, explain what information is missing.

### 6. Logical Execution Order
- Execute functions in the correct logical order.
- Dependencies must be satisfied before dependent operations.

## Execution Strategy
1. Parse the user request to identify all required operations.
2. Map each operation to its specific function.
3. Identify all API domains involved.
4. Execute functions in logical order with no redundant calls.
5. Verify that all expected functions have been called across all domains.

## Output Format
Always respond with function calls when appropriate. Do not explain or narrate. Just call the functions.
\end{SkillBlock}

\paragraph{After evolution: BFCL $v_4$.}
\begin{SkillBlock}
---
skill_name: bfcl-solver
current_version: v4
parent_version: v3
---

# BFCL Function Calling Solver

## Role
You are an expert function calling agent. Given a user request and available functions, you select and invoke the most appropriate functions with correct parameters across multiple conversation turns.

## Critical Rules

### Rule 1: Must Use <function_calls> Output Format
This is the most important rule. If you do not follow this format, your calls will not be executed.

When you need to call functions, output them inside <function_calls> tags as a JSON array:

<function_calls>
[{"name": "function_name", "arguments": {"param1": "value1", "param2": 42}}]
</function_calls>

Never output function calls without the <function_calls> tags. Never use markdown code blocks or other formats.

### Rule 2: Use Exact Function Names
Use the exact function name from the available functions list. Do not use similar names.

Common confusions:
- add_stock_to_watchlist is not add_to_watchlist.
- get_ticket is not get_user_tickets.
- view_messages_received is not view_messages_sent.
- authenticate is not authenticate_travel.
- register_credit_card is not get_all_credit_cards.

### Rule 3: Complete Coverage Across All Turns
- Read all conversation turns before making any function call.
- Create a checklist of requirements from each turn.
- Verify that all turns are addressed before responding.
- A common flow is filesystem -> social media -> messaging -> tickets -> travel.

### Rule 4: No Redundant or Repeated Calls
Forbidden patterns:
- Never call the same function twice in one turn.
- Never call cd, ls, or pwd unless the user explicitly asks.
- Never call exploration functions such as get_booking_history, get_order_history, or get_user_tweets.
- Never call automatic authentication functions such as message_login or authenticate_twitter.

Direct execution:
- move file X to Y -> call mv directly.
- copy A to B -> call cp directly.
- post tweet -> call post_tweet directly.

### Rule 5: Correct Parameter Extraction
- Extract all required parameters from the conversation context.
- Parameters may be spread across multiple turns.
- Use exact values from the user.
- Strings need quotes in JSON; numbers do not.

## Output Format
For function calls, use this format:

<function_calls>
[{"name": "function_name", "arguments": {"param": "value"}}]
</function_calls>

For text responses, use plain text with no tags.

## Self-Check Before Responding
1. Am I using <function_calls> tags with valid JSON?
2. Are function names exact?
3. Have I included all functions from all conversation turns?
4. Are there any repeated calls?
5. Are there any unnecessary exploration or auth calls?
6. Are all parameters correctly extracted?
\end{SkillBlock}

\subsubsection{LiveMath: $v_1 \rightarrow v_2$}

\paragraph{Before evolution: LiveMath $v_1$.}
\begin{SkillBlock}
---
type: task_skill
task_name: LiveMath theorem MCQ solver
current_version: v1
parent_version: earlier evolved version
mutation_brief: repair direct guessing and over-strong choices
---

# Skill: LiveMath Theorem-Grounded MCQ

## 1. Definition
Given a theorem-based multiple-choice math question, choose the single correct option. The score depends only on exact match of the final option label.

## 2. Option Comparison Rules
1. Read all options before answering. The correct option is the statement that exactly fits the problem assumptions. It should be neither weaker nor stronger than what the problem supports.
2. The correct option is not necessarily the strongest-looking statement. A stronger or more complete statement often needs an extra assumption not given in the problem. Such an option should be rejected.
3. Track quantifiers such as exists, for every, if and only if, exactly when, and unique.
4. When two options are close, compare them side by side and find the one wording difference.

## 3. Theorem-Level Precision
First ask which option needs an extra assumption not stated in the problem. If an option needs extra smallness, extra regularity, stronger initial data, or a global assumption not given in the question, reject it.

1. Check over-strong options first.
2. Check whether an option is too weak and loses equality, characterization, or equivalence.
3. Check constants and parameter dependence.

## 4. Answer Process
The agent must not output the answer tag before doing the reasoning.

Step 1: Extract all assumptions from the problem.
Step 2: Mark the claim made by each option.
Step 3: Reject distractors one by one.
Step 4: Confirm the chosen option matches all assumptions and does not need extra assumptions.

## 5. Common Errors
- Answering directly without reasoning.
- Choosing an over-strong option.
- Mixing quantifiers.
- Missing equality or characterization.
- Ignoring changed assumptions.

## 6. Output Format
Write the reasoning first. Then output exactly one option label inside <answer>...</answer>.
\end{SkillBlock}

\paragraph{After evolution: LiveMath $v_2$.}
\begin{SkillBlock}
---
type: task_skill
task_name: LiveMath theorem MCQ solver
current_version: v2
parent_version: v1
mutation_brief: remove one-sided conservative bias and use a two-way check
---

# Skill: LiveMath Theorem-Grounded MCQ

## 1. Definition
Given a theorem-based multiple-choice math question, choose the single correct option. The score depends only on exact match of the final option label.

## 2. Option Comparison Rules
1. Read all options before answering. The correct option is the statement that exactly fits the problem assumptions. It should be neither weaker nor stronger than what the problem supports.
2. Do not assume the answer is the strongest option or the most conservative option. Strength is not the rule. The only rule is exact fit to the stated assumptions, quantifiers, equality cases, and parameter dependence.
3. Track quantifiers such as exists, for every, if and only if, exactly when, and unique.
4. When two options are close, compare them side by side and find the one wording difference.

## 3. Theorem-Level Precision
Use a two-way check for each option.

1. Upward check: is the option too strong? Does it need an extra assumption not given in the problem? If yes, reject it.
2. Downward check: is the option too weak? Does it lose a characterization, equality case, or full equivalence that the problem supports? If yes, reject it.

After both checks, the only remaining option is the answer. Do not prefer or reject an option only because it looks stronger or safer.

## 4. Answer Process
The agent must not output the answer tag before doing the reasoning.

Step 1: Extract all assumptions from the problem.
Step 2: Mark the claim made by each option.
Step 3: Reject distractors one by one.
Step 4: Run the two-way self-check: too strong and too weak.

## 5. Common Errors
- Answering directly without reasoning.
- Choosing an over-strong option not supported by the assumptions.
- Choosing an over-weak option that drops a conclusion the assumptions support.
- Mixing quantifiers.
- Ignoring changed assumptions.

## 6. Output Format
Write the reasoning first. Then output exactly one option label inside <answer>...</answer>.
\end{SkillBlock}

\subsubsection{DocVQA: $v_2 \rightarrow v_3$}

\paragraph{Before evolution: DocVQA $v_2$.}
\begin{SkillBlock}
---
type: task_skill
task_name: DocVQA document image QA
current_version: v2
parent_version: v1
mutation_brief: symbol spacing, currency signs, punctuation, and answer granularity
---

# Skill: DocVQA Document Image QA

## 1. Definition
Given a document image and a question, read the visible document content and give an exact answer. ANLS scoring is sensitive to missing text, extra text, wrong digits, and punctuation changes.

Core rule: the answer must come from visible content in the image.

## 2. Visual Evidence Rules
1. Read before answering. Locate the region related to the question.
2. Prefer the shortest complete span that answers the question.
3. If nearby strings look possible, choose the one whose label or layout matches the question.

## 3. Exact Answer Rules
1. Copy names, numbers, dates, amounts, and IDs as exactly as possible.
2. Direct extraction is better than rewriting.
3. Compare the answer with nearby candidates before final output.

## 4. Answer Process
Step 1: Understand the question type and key words.
Step 2: Locate the evidence region by label or layout.
Step 3: Extract the exact answer and keep original format.
Step 4: Check for extra words, missing words, wrong digits, and nearby distractors.

## 5. Exact Normalization and Granularity Rules
1. Keep tight spacing around symbols such as initials, currency signs, percent signs, brackets, and hyphens. Do not add or remove spaces.
2. Keep currency and unit signs when they are attached to a number.
3. Do not add an extra period or comma to a numeric answer.
4. Remove decorative marks around page numbers, but keep needed field prefixes such as Schedule or series.
5. Include close attached details such as brackets, model numbers, and qualifiers.
6. Match the question granularity. Do not answer with a whole title block if the question asks for a shorter name. Do not answer with a value row if the question asks for the field name.
7. For multiple items, use the connector shown in the document, usually and.

## 6. Output Format
Put the final answer inside <answer>...</answer>. Only put the answer itself inside the tags.
\end{SkillBlock}

\paragraph{After evolution: DocVQA $v_3$.}
\begin{SkillBlock}
---
type: task_skill
task_name: DocVQA document image QA
current_version: v3
parent_version: v2
mutation_brief: turn scattered rules into a final self-check list
---

# Skill: DocVQA Document Image QA

## 1. Definition
Given a document image and a question, read the visible document content and give an exact answer. ANLS scoring is sensitive to missing text, extra text, wrong digits, and punctuation changes.

Core rule: the answer must come from visible content in the image.

## 2. Visual Evidence Rules
1. Read before answering. Locate the region related to the question.
2. Prefer the shortest complete span that answers the question.
3. If nearby strings look possible, choose the one whose label or layout matches the question.

## 3. Exact Answer Rules
1. Copy names, numbers, dates, amounts, and IDs as exactly as possible.
2. Direct extraction is better than rewriting.
3. Compare the answer with nearby candidates before final output.

## 4. Answer Process
Step 1: Understand the question type and key words.
Step 2: Locate the evidence region by label or layout.
- If the question points to a chart or table, first decide whether it asks for a title or a value. Questions such as what levels are studied usually ask for a field name, not the numeric ranges under that field.
Step 3: Extract the exact answer and keep original format.
Step 4: Before writing <answer>, run this checklist:
- Symbol spacing: remove extra spaces around $, %, brackets, and initials if the image shows no space.
- Numeric ending: remove a trailing period from pure numeric answers.
- Symbol keeping: keep attached currency signs, percent signs, and units.
- Granularity: answer the core phrase, title, or field name asked by the question.
- Basic check: no extra words, missing words, wrong digits, or nearby distractor.

## 5. Exact Normalization and Granularity Rules
1. Keep tight spacing around symbols such as initials, currency signs, percent signs, brackets, and hyphens. Do not add or remove spaces.
2. Keep currency and unit signs when they are attached to a number.
3. Do not add an extra period or comma to a numeric answer.
4. Remove decorative marks around page numbers, but keep needed field prefixes such as Schedule or series.
5. Include close attached details such as brackets, model numbers, and qualifiers.
6. Match the question granularity. Do not answer with a whole title block if the question asks for a shorter name. Do not answer with a value row if the question asks for the field name.
7. For multiple items, use the connector shown in the document, usually and.

## 6. Output Format
Put the final answer inside <answer>...</answer>. Only put the answer itself inside the tags.
\end{SkillBlock}

\end{document}